\documentclass[lettersize,journal]{IEEEtran}
\ifCLASSOPTIONcompsoc
  \usepackage[nocompress]{cite}
\else
  \usepackage{cite}
\fi
\ifCLASSINFOpdf
\else
\fi
\usepackage{amsmath,amsfonts}
\usepackage{algorithmic}
\usepackage{algorithm}
\usepackage{array}
\usepackage{widetext}
\usepackage[caption=false,font=normalsize,labelfont=sf,textfont=sf]{subfig}
\usepackage{textcomp}
\usepackage{stfloats}
\usepackage{url}
\usepackage{verbatim}
\usepackage{graphicx}
\usepackage{cite}
\usepackage{notoccite}
\usepackage{amsthm,amsmath,amssymb}
\usepackage{mathrsfs}
\usepackage{soul}
\usepackage{color, xcolor}
\usepackage{stfloats}
\usepackage{bm}
\usepackage{amsthm} 

\newtheoremstyle{bfnote}
                {} 
                {} 
                {} 
                {1em} 
                {\itshape} 
                {:} 
                { } 
                {} 
\theoremstyle{bfnote}
\newtheorem{theorem}{Theorem}
\newtheorem{conjecture}{Conjecture}
\newtheorem{definition}{Definition}
\newtheorem{assumption}{Assumption}

\newtheorem{corollary}{Corollary}
\newtheorem{lemma}{Lemma}

\makeatletter
\renewenvironment{proof}[1][\proofname]{\par
  \pushQED{\qed}%
  \normalfont \topsep6\p@\@plus6\p@\relax
  \trivlist
  \item[\hskip2.1em
        \itshape
    #1\@addpunct{:}]\ignorespaces
}{%
  \popQED\endtrivlist\@endpefalse
}
\makeatother

\allowdisplaybreaks[4]

\begin{document}

\title{A Dynamics Theory of Implicit Regularization in Deep Low-Rank Matrix Factorization}

\author{Jian Cao, Chen Qian, Yihui Huang, Dicheng Chen, Yuncheng Gao, Jiyang Dong, Di Guo, Xiaobo Qu*

\thanks{This work is partially supported by National Natural Science Foundation (62122064, 61971361 and 61871341), Natural Science Foundation of Fujian Province of China (2021J011184), President Fund of Xiamen University (20720220063), and the Xiamen University Nanqiang Outstanding Talents Program.  }  
\thanks{Jian Cao, Chen Qian, Yihui Huang, Dicheng Chen, Yuncheng Gao, Jiyang Dong and Xiaobo Qu* are with the Department of Electronic Science, Biomedical Intelligent Cloud R\&D Center, Fujian Provincial Key Laboratory of Plasma and Magnetic Resonance, Xiamen University, Xiamen 361005, China. (*Corresponding author, e-mail: quxiaobo@xmu.edu.cn)}
\thanks{Di Guo is with School of Computer and Information Engineering, Fujian Engineering Research Center for Medical Data Mining and Application, Xiamen University of Technology, Xiamen 361024, China.}}

\markboth{Journal of \LaTeX\ Class Files,~Vol.~X, No.~X, X~XXXX}%
{Shell \MakeLowercase{\textit{et al.}}: Bare Advanced Demo of IEEEtran.cls for IEEE Computer Society Journals}

\maketitle

\begin{abstract}

Implicit regularization is an important way to interpret neural networks. Recent theory starts to explain implicit regularization with the model of deep matrix factorization (DMF) and analyze the trajectory of discrete gradient dynamics in the optimization process. These discrete gradient dynamics are relatively small but not infinitesimal, thus fitting well with the practical implementation of neural networks. Currently, discrete gradient dynamics analysis has been successfully applied to shallow networks but encounters the difficulty of complex computation for deep networks. In this work, we introduce another discrete gradient dynamics approach to explain implicit regularization, i.e. landscape analysis. It mainly focuses on gradient regions, such as saddle points and local minima. We theoretically establish the connection between saddle point escaping (SPE) stages and the matrix rank in DMF. We prove that, for a rank-$R$ matrix reconstruction, DMF will converge to a second-order critical point after $R$ stages of SPE. This conclusion is further experimentally verified on a low-rank matrix reconstruction problem. This work provides a new theory to analyze implicit regularization in deep learning.

\end{abstract}


\section{Introduction}

\IEEEPARstart{D}{eep} learning has made a great breakthrough in many fields, e.g. computer vision \cite{RN85,RN110,RN111}, natural language processing \cite{RN87,RN227, RN113}, time-series forecasting \cite{RN115,RN118,RN88}, biomedicine \cite{RN122, RN228} and biology \cite{RN100,RN121}. Basically, a fully connected neural network $h:\mathbb{R}^{n_{x}}\rightarrow\mathbb{R}^{n_{y}}$ obtains an optimal solution by minimizing a loss function $\ell$, i.e., 
\begin{align}
\label{nnloss}
\mathop{\min}_{\Theta}\frac{1}{2M} \sum_{m=1}^{M}\ell(h(\Theta,\mathbf{x}_{m}), \mathbf{y}_{m}),
\end{align}
where $\mathbf{x}_m\in \mathbb{R}^{n_{x}}$ denotes the input, $\mathbf{y}_m\in \mathbb{R}^{n_{y}}$ is the label and $\Theta$ is a set of trainable network parameters. To solve (\ref{nnloss}), a representative approach \cite{RN133} is to use the gradient descent with a learning rate $\eta$ according to
\begin{align}
\label{nnlr}
\Theta(t+1)=\Theta(t)-\eta\cdot\frac{1}{2M}\sum_{m=1}^{M}\nabla\ell(h(\Theta(t),\mathbf{x}_{m}), \mathbf{y}_{m}),
\end{align}
where $\Theta(t)$ are trainable parameters at the $t^{\text{th}}$ iteration.

To theoretically explain the generalization ability of neural networks \cite{RN59,RN229,RN60,RN63,RN62,RN92}, one recent approach is using implicit regularization \cite{RN93} as follows
\begin{align}
\mathop{\min}_{\Theta}\frac{1}{2M} \sum_{m=1}^{M}\ell(h(\Theta,\mathbf{x}_{m}), \mathbf{y}_{m})+G(\Theta),
\end{align}
where $G(\Theta)$ is an implicit regularization term. However, understanding implicit regularization is non-trival since the neural networks are usually nonlinear. Thus, many current research tends to linear networks as a starting point for theoretical analysis \cite{RN41,RN40,RN52}.

For linear neural networks $h_{\text {Linear}}(\mathbf{x}_{m})=\mathbf{W}_{L}\mathbf{W}_{L-1} \cdots$ $\mathbf{W}_{1}\mathbf{x}_{m}$, the standard supervised learning is to solve 
\begin{align}
\label{loss}
\min_{\substack{\mathbf{W}_{l}\in\mathbb{R}^{n_{l} \times n_{l-1}} \\ l=1,2,\cdots,L}} \frac{1}{2M} \sum_{m=1}^{M}\left\|\mathbf{W}_{L} \mathbf{W}_{L-1} \cdots\mathbf{W}_{1} \mathbf{x}_{m}-\mathbf{y}_{m}\right\|^{2},
\end{align}
where $\mathbf{W}_{l} \in \mathbb{R}^{n_{l} \times n_{l-1}}$ is the weight matrix, $l=1,2,\cdots,L$. This linear structure is relatively simple but can also exhibit nonlinear learning phenomena and non-convexity  \cite{RN43,RN45}, thus it still deserves attention. Specifically, the non-convexity of the linear network is introduced by the coupling between the weight matrices \cite{RN43,RN45}. By treating all the inputs $\{\mathbf{x}_{m}\}^{M}_{m=1}$ as an initial matrix $\mathbf{W}_{0}$, (\ref{loss}) can be modelled as a deep matrix factorization (DMF) problem \cite{RN40}, which learns the low-rank mapping in a self-supervised manner. Up to now, DMF has become a valuable way to analyze implicit regularization \cite{RN40,RN41,RN84,RN98,RN97,RN99,RN134}.

To theoretically explain the implicit regularization in linear neural networks, gradient dynamics have been adopted to analyze the learning process in minimizing the loss function. Depending on whether the learning rate $\eta$ in (\ref{nnlr}) is infinitesimal or not, gradient dynamics can be categorized into continuous (infinitesimally small) or discrete (non-infinitesimal) forms (Fig. \ref{Summary}). The former has made some theoretical progress which indicates that implicit regularization is a low-rank bias \cite{RN43,RN40,RN45, RN4,RN34}, but departs from real implementation since the practical learning rate is usually non-infinitesimal. Here, we mainly focus on the discrete gradient dynamics for implicit regularization.

A representative discrete gradient dynamics approach is trajectory analysis \cite{RN52,RN44,RN46,RN41}. It discusses the trajectory (solid line in Fig. \ref{Visualization}(a)) generated by the gradient of trainable parameters, i.e. weight matrices. However, discrete gradient dynamics imposes a computational tax for trajectory analysis, and the theoretical research is limited to two-layer networks \cite{RN52,RN44, RN41}.  Meanwhile, some research under continuous gradient dynamics shows the nonlinear learning phenomena in the loss evolution \cite{RN43,RN34}, i.e. plateau and rapid decline stages (Fig. \ref{Visualization}(b)). As illustrated in Fig. \ref{Visualization}(c), the plateau stage with a small gradient indicates that parameters get stuck at saddle points, and the rapid decline stage represents a large gradient region \cite{RN43}. These inspire us to conjecture that saddle points in gradient regions are the key to understanding implicit regularization.

\begin{figure*}[ht]
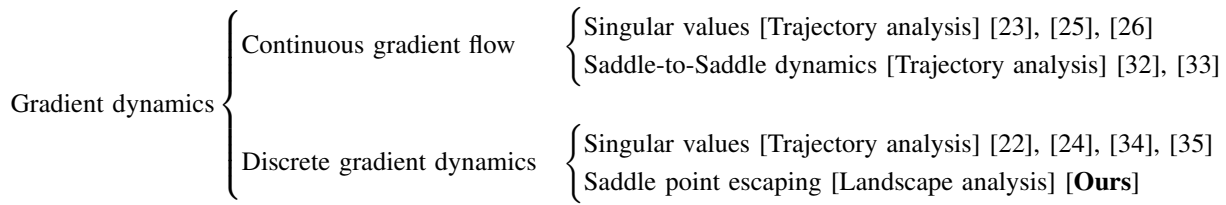
 
\begin{equation*}
\text{Gradient dynamics}
\begin{cases}
    \text{Continuous gradient flow}&
     \begin{cases}
        \text{Singular values [Trajectory analysis] \cite{RN43,RN40,RN45}}\\        
        \text{Saddle-to-Saddle dynamics [Trajectory analysis] \cite{RN4, RN34}}
    \end{cases}\\
    \\
    \text{Discrete gradient dynamics}&
      \begin{cases}
       \text{Singular values [Trajectory analysis] \cite{RN52,RN44,RN46,RN41}}\\
       \text{Saddle point escaping [Landscape analysis] [\textbf{Ours}]}
      \end{cases}
\end{cases}
\end{equation*}
\caption{Summary of gradient dynamics approaches in implicit regularization.} 
\label{Summary}
\end{figure*}

Landscape analysis is widely used for convergence properties, which directly discusses some gradient regions on the loss landscape \cite{RN21}. The loss landscape is the surface given in Fig. \ref{Visualization}(a) and the corresponding gradient values are shown in Fig. \ref{Visualization}(c). The points with the zero gradient value can be defined as first-order critical points \cite{RN57}, such as local minima and saddle points. Among the research on saddle points in shallow networks, all local minima are global minima and all saddle points are strict \cite{RN23,RN21}. Adaptive gradient descent (including RMSProp \cite{RN128}) can help networks escape from saddle points quickly \cite{RN38}. For deep networks, every critical point is a global minimum or a saddle point \cite{RN31,RN65}, whereas the existence of non-strict saddle points makes it hard to analyse convergence \cite{RN31}. To the best of our knowledge, the landscape has not been applied to analyze implicit regularization.
\begin{figure}[!t]
\centering
\includegraphics[width=3.45in]{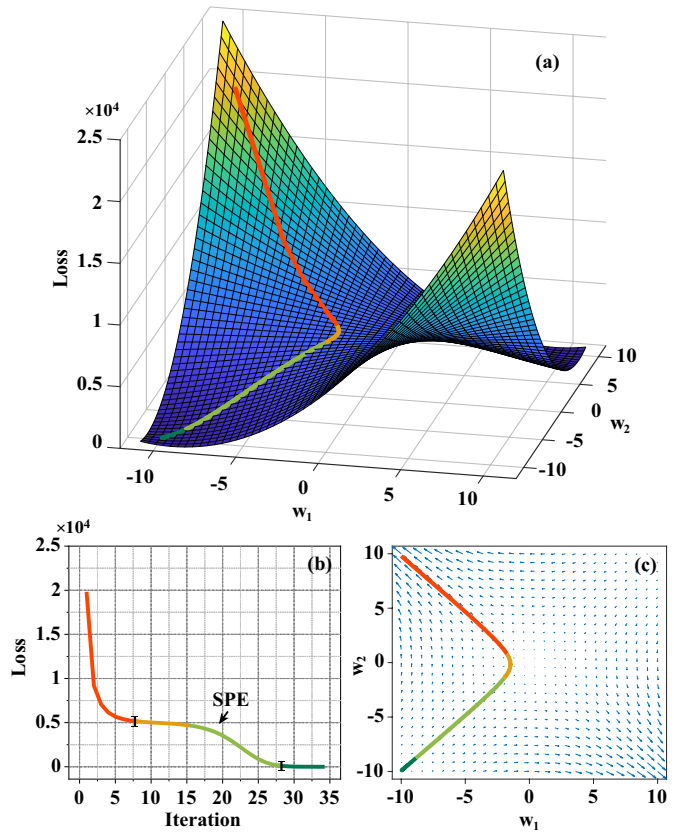}
\caption{Visualization evolution of loss for a toy example. (a) Gradient dynamics trajectory and loss landscape; (b) Evolution process of loss over iterations; (c) Gradient dynamics trajectory and gradient vector field. Note that the toy example for minimizing $\rm{Loss}\!=\!\frac{1}{2}(100-\mathbf{w}_{1}\mathbf{w}_{2})^2$ converges to a global minimum under gradient descent. Among them, the first line (red) and the third line (light green) represent rapid decline stages with a large gradient. The second line (yellow) represents a plateau stage with a small gradient indicating that the dynamics trajectory gets stuck at saddle points. The fourth line (dark green) represents a global minimum and its gradient value is small. A saddle point escaping (SPE) stage consists of a plateau stage and a rapid decline stage.}
\label{Visualization}
\end{figure}

In this work, we focus on a theoretical illustration of implicit regularization in DMF based on landscape analysis. We first define a saddle point escaping (SPE) stage consisting of a plateau stage and a rapid decline stage, and introduce the impact of increasing learning rates on the SPE stages. Further, we build the relationship between implicit regularization and the number of saddle point escaping (SPE) stages under rank-$R$ matrix reconstruction in DMF. We will theoretically prove that under discrete gradient dynamics, DMF converges to a second-order critical point after $R$ stages of SPE. In addition, the theoretical result is verified by the low-rank matrix reconstruction experiments. 

Compared to \cite{RN4, RN34}, we all focus on the relationship between the SPE stages and implicit regularization. The SPE stages are conjectured as the Saddle-to-Saddle dynamics in \cite{RN4, RN34}, and they only prove the first step of the Saddle-to-Saddle dynamics under gradient flow. Our technical work emphasizes the impact of discrete gradient dynamics on the SPE stages, further understanding implicit regularization through the number of iterations spent in the SPE stages.

The rest of the article is organized as follows. Section \uppercase\expandafter{2} introduces the deep matrix factorization model and related parameter initialization methods. Section \uppercase\expandafter{3} presents a theoretical conjecture based on the experimental phenomenon, and further proves it theoretically. Section \uppercase\expandafter{4} shows the experimental results and conclusions are made in Section \uppercase\expandafter{5}. 


\section{Preliminaries and Model}

We denote that $\|\cdot\|$ is the Euclidean norm of a vector or the spectral norm of a matrix, and $\|\cdot\|_{F}$ is the Frobenius norm. $\lambda_{\max}(\mathbf{A})$, $\lambda_{\min}(\mathbf{A})$ and $\lambda_{i}(\mathbf{A})$ are denoted by maximum, minimum and $i$-th largest eigenvalue of a matrix $\mathbf{A}$. $\sigma_{\max}(\mathbf{A})$, $\sigma_{\min}(\mathbf{A})$ and $\sigma_{i}(\mathbf{A})$ are also denoted by maximum, minimum and $i$-th largest singular value of $\mathbf{A}$. $\nabla\ell(\mathbf{W})$ and $\nabla^2\ell(\mathbf{W})$ are gradient and Hessian matrix of $\ell$ at the point $\mathbf{W}$. Meanwhile, the sampling rate is expressed as the proportion of observed items in all data and we can obtain an incomplete matrix (Fig. \ref{lrrr}(a)) for reconstruction experiments. One iteration represents the updating process of model parameters using full-batch learning.

Given $n_0=M$, the samples $\mathbf{x}_1,\cdots,\mathbf{x}_{M}$ span the input space $\mathbf{W}_{0} \in \mathbb{R}^{{n_{0}} \times {n_{0}}}$. We set $\mathbf{W}_{0}=\mathbf{I}$ and $\mathbf{W}_{0}$ is not updated by gradient descent. Thus, DMF model \cite{RN40, RN46} is defined as
\begin{align}
\label{DMF}
\mathbf{W}=\mathbf{W}_{L} \mathbf{W}_{L-1} \!\cdots\! \mathbf{W}_{1},
\end{align}
where $L$ is the network depth. The product matrix $\mathbf{W}$ is uniquely determined by $\mathbf{W}^{*}$. (\ref{loss}) converts into the loss for DMF with self-supervised learning

\begin{align}
\label{lossdmf}
\phi\left(\mathbf{W}_{1},\cdots,\mathbf{W}_{L}\right)=\ell(\mathbf{W})=\frac{1}{2}\left\|\mathbf{W}-\mathbf{W}^{*}\right\|_{F}^{2}.
\end{align}
   

We solve DMF model with the full-batch RMSProp \cite{RN128} (Algorithm \ref{alg1}), meaning that $\mathbf{W}_{l}$ obeys the discrete gradient dynamics as follows
\begin{equation}
\begin{aligned}
&\mathbf{W}_{l}(t\!+\!1)\!=\!\mathbf{W}_{l}(t)-\eta\mathbf{A}(t)\frac{\partial \phi}{\partial \mathbf{W}_{l}}(\mathbf{W}_{1},\cdots,\mathbf{W}_{L})\\
&=\mathbf{W}_{l}(t)-\eta\mathbf{A}(t)\mathbf{W}_{l+1:L}^{\top}(t) \nabla \ell(\mathbf{W}(t)) \mathbf{W}_{1:l-1}^{\top}(t),
\end{aligned}
\end{equation}
where $\mathbf{A}(t)$ uses an exponential moving average (EMA) for past gradient information to achieve adaptivity of learning rate, $\nabla \ell(\mathbf{W}(t))=\mathbf{W}(t)-\mathbf{W}^{*}$, $\mathbf{W}_{l+1:L}\!=\!\mathbf{W}_{L}\!\cdots\!\mathbf{W}_{l+1}$ and  $\mathbf{W}_{1:l-1}\!=\!\mathbf{W}_{l-1}\!\cdots\!\mathbf{W}_{1}$. In each iteration, we can obtain the discrete gradient dynamics of $\mathbf{W}$ as follows
\begin{equation}
\begin{aligned}
\label{equ:W}
\mathbf{W}(t\!+\!1)&=\prod_{l=L}^{1}\Big(\mathbf{W}_{l}(t)\!-\!\eta\mathbf{A}(t)\frac{\partial \phi}{\partial \mathbf{W}_{l}}(\mathbf{W}_{1},\cdots,\mathbf{W}_{L})\Big)\\
&=\mathbf{W}(t)-\eta\mathbf{A}(t)\nabla \mathbf{L}(t)+\mathbf{E}(t),
\end{aligned}
\end{equation}
where 
\begin{equation}
\begin{aligned}
\nabla \mathbf{L}(t)=\sum_{l=1}^{L}\Big(\mathbf{W}_{l+1:L}&(t)\mathbf{W}_{l+1:L}^{\top}(t)\\
&\nabla \ell(\mathbf{W}(t))\mathbf{W}_{1:l-1}^{\top}(t)\mathbf{W}_{1:l-1}(t)\Big),
\end{aligned}
\end{equation}
and $\mathbf{E}(t)$ denotes higher order terms.

For the initialization of $\mathbf{W}_{l}$, the balanced initialization \cite{RN19} ensures that all $\mathbf{W}_{l}$ have the same non-zero singular values and removes the term $\mathbf{E}(t)$ from (\ref{equ:W}). This initialization is further extended to approximate form \cite{RN14} to retain the full gradient information. In this paper, we use approximate balanced initialization, defined as follows
\begin{definition}
For $\vartheta \geq 0$, the matrix $\mathbf{W}_{l}\in \mathbb{R}^{n_{l}\times n_{l-1}}$ satisfies approximate balanced initialization, denoted as $\vartheta$-balanced if:
\begin{equation}
\begin{aligned}
\| \mathbf{W}^{T}_{l+1}\mathbf{W}_{l+1}\!-\!\mathbf{W}_{l}\mathbf{W}^{T}_{l}\|_{F}\leq \vartheta, \forall l\in\{1,\cdots,L-1\}.
\end{aligned}
\end{equation}
\end{definition}

\begin{algorithm}[H]
\caption{Solving DMF problem with the full-batch RMSProp}\label{alg1}
\begin{algorithmic}[1]
\REQUIRE initial $\mathbf{W}_l(0)$, $l\!=\!1,\!\cdots\!,L$, total number of iterations $T$, learning rate $\eta$, $V(0)\gets0$, $\alpha\gets0.99$
\STATE \textbf{for} $t=0,\cdots,T$ \textbf{do}
\STATE \hspace{0.3cm}$ V(t\!+\!1)\!\gets\!\alpha V(t)+(1-\alpha)\!\sum_{l=1}^{L}\!\| \frac{\partial \phi}{\partial\mathbf{W}_{l}}(\mathbf{W}_{1},\!\cdots\!,\mathbf{W}_{L})\|_{F}^{2}$
\STATE \hspace{0.3cm}$\mathbf{A}(t)\gets\dfrac{1}{\sqrt{\frac{V(t+1)}{1-\alpha^{t+1}}}+\varepsilon}\cdot\mathbf{I}$
\STATE \hspace{0.3cm}$\mathbf{W}_l(t+1)\gets \mathbf{W}_l(t)-\eta \mathbf{A}(t)\frac{\partial \phi}{\partial \mathbf{W}_{l}}(\mathbf{W}_{1},\cdots,\mathbf{W}_{L})$
\STATE \hspace{0.3cm}$\mathbf{W}(t+1)\gets \mathbf{W}_L(t+1)\mathbf{W}_{L-1}(t+1)\cdots \mathbf{W}_{1}(t+1)$
\STATE \textbf{end for}
\end{algorithmic}
\end{algorithm}
\section{Dynamical Analysis and Implicit Regularization}

In this section, we describe a theoretical conjecture about implicit regularization based on experimental phenomena, and further prove it through landscape analysis.

\begin{figure}[!t]
\centering
\includegraphics[width=3in]{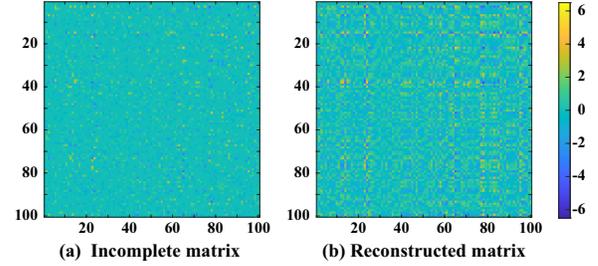}
\caption{Incomplete matrix reconstruction under 30\% sampling rate. (a) Incomplete matrix; (b) Reconstructed matrix. Note that a 100$\times$100 random matrix with rank-6 is reconstructed by DMF of depth-6, $n_{l}$ = 100 ($l=1,\cdots,6$). The standard deviation of the initialization and learning rate are both $10^{-3}$. For visualization, the matrix is shown by colormap.}
\label{lrrr}
\end{figure}

\begin{figure}[!t]
\centering
\includegraphics[width=2.6in]{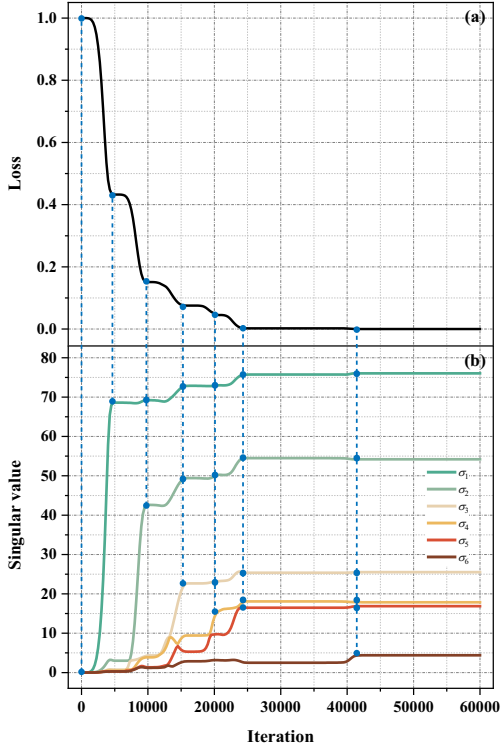}
\caption{Dynamic performance of DMF using RMSProp. (a) Evolution of the loss; (b) Evolution of singular values. Note that a 100$\times$100 random matrix with rank-6 is reconstructed by DMF of depth-6, $n_{l}$ = 100 ($l=1,\cdots,6$). The standard deviation of the initialization and learning rate are both $10^{-3}$, and the sampling rate is 30\%.}
\label{fig_sv&loss}
\end{figure}

\subsection{Gradient Dynamics in Deep Matrix Factorization} 
To verify the nonlinear learning phenomenon of the toy example in Fig. \ref{Visualization}, a low-rank matrix reconstruction experiment in DMF with RMSProp is conducted (Fig. \ref{lrrr}). In the iterations, the value of the loss function, singular values of the reconstructed matrix, and the learning rate are shown in Fig. \ref{fig_sv&loss} and Fig. \ref{fig_lr}.

Several alternating stages of plateaus and rapid decline in the loss evolution (Fig. \ref{fig_sv&loss}(a)). The plateau stage with a small gradient is defined as the saddle point region while the rapid decline stage represents a large gradient region \cite{RN43,RN34}. We call the successive plateaus and rapid decline stages as one stage of saddle point escaping (SPE).

At each stage of SPE, singular values (Fig. \ref{fig_sv&loss}(b)) of the reconstructed matrix evolve and influence each other, implying that singular values (or the matrix rank, i.e. the number of non-zero values) could represent the learning dynamics \cite{RN44}. When the smaller singular value starts to learn, it will cause perturbation to the larger singular value. This means that the optimization process generates a new solution during each perturbation period, while these periods correspond to each SPE stage as shown in Fig. \ref{fig_sv&loss}. An interesting observation is that the number of SPE stages is equal to the rank of the reconstructed low-rank matrix. Thus, we conjecture that this experiment describes implicit regularization in DMF. Namely, DMF can converge to a second-order critical point after the iteration number in rank times of SPE stages, which demonstrates the network has a low-rank constraint ability (Conjecture 1).

\begin{definition}[\emph{Second-order critical point}]
A $(\tau_{g},\tau_{h})$-critical point of $\ell$ is a point $\mathbf{W}$ so that $\|\nabla \ell(\mathbf{W})\|\leq \tau_{g}$ and $\lambda_{\min}(\nabla^2 \ell(\mathbf{W}))\geq -\tau_{h}$, where $\tau_{g}$, $\tau_{h} > 0$.
\end{definition}

\begin{conjecture}
For reconstructing a rank-$R$ matrix in DMF, there are $R+1$ first-order critical points $\mathbf{W}^{0}$, $\mathbf{W}^{1}$, $\cdots$, $\mathbf{W}^{R}$ and $R$ stages of SPE $\Psi^{1}$, $\Psi^{2}$, $\cdots$, $\Psi^{R}$ connecting these points, that is $\lim_{t \rightarrow \infty} \Psi^{i}(t)=\mathbf{W}^{i}$, $\lim_{t \rightarrow -\infty} \Psi^{i}(t)=\mathbf{W}^{i-1}$, $i=1,2,\cdots,R$. With high probability, the discrete gradient dynamics (\ref{equ:W}) converges to a second-order critical point $\mathbf{W}^{R}$ after $T=O(R\cdot t_{\rm{SPE}})$ iterations.
\end{conjecture}

Moreover, we focus on the learning rate. According to Fig. \ref{fig_lr}, the learning rate increases rapidly on the plateau and becomes small at the rapid decline stage. It shows that periodically using a larger learning rate $r$ can escape saddle points \cite{RN37}
\begin{equation}
\begin{aligned}
\label{lglr}
\mathbf{W}_l(t+1) = \mathbf{W}_l(t)-r\mathbf{A}(t)\frac{\partial \phi}{\partial \mathbf{W}_{l}}(\mathbf{W}_{1},\!\cdots\!,\mathbf{W}_{L}).
\end{aligned}
\end{equation}
However, in some cases, RMSProp cannot converge to the second-order critical point \cite{RN51}. To solve this problem, the preconditioner $\hat{\mathbf{A}}(t)$ of an idealized setting is proposed to guarantee convergence \cite{RN38}. Specifically, $\mathbf{A}(t)$ is proven to be close to $\hat{\mathbf{A}}(t)$ after $T_{\rm{burnin}}$ iterations, namely $\|\mathbf{A}(t)-\hat{\mathbf{A}}(t)\|\leq \Delta$, which guarantee second-order convergence properties for any adaptive gradient descent methods including the RMSProp (Theorem 4.1 in \cite{RN38}).

Therefore, to ensure the second-order convergence of Algorithm \ref{alg1}, we introduce this RMSProp \cite{RN38} into solving the DMF problem. We periodically increase the learning rate in (\ref{lglr}) and adopt a suitable preconditioner $\hat{\mathbf{A}}(t)$ from Definition 3, which yields Algorithm \ref{alg2}.

\begin{definition}
\label{def:Ahat}
We say $\hat{\mathbf{A}}(t)$ is $(\Lambda, \nu, \lambda_{-})$-preconditioner if, for all $\mathbf{W}_l$, the following bounds hold. First, $\|\hat{\mathbf{A}}\|\leq\Lambda$. Second,  $\nu \leq\lambda_{\min}(\hat{\mathbf{A}}\nabla \mathbf{L}\nabla \mathbf{L}^{\top}\hat{\mathbf{A}}^{\top})$. Third, $\lambda_{-}\leq \lambda_{\min}(\hat{\mathbf{A}})$.
\end{definition}

\begin{algorithm}[H]
\caption{Solving DMF problem with the increasing learing rate full-batch RMSProp}
\label{alg2}
\begin{algorithmic}[1]
\REQUIRE initial $\mathbf{W}_l(0)$, $l\!=\!1,\!\cdots\!,L$, total number of iterations $T$, learning rates $\eta$ and $r$, threshold $t_{\rm{SPE}}$, $V(0)\gets V(T_{\rm{burnin}})$, $\alpha\gets0.99$
\STATE \textbf{for} $t=0,\cdots,T$ \textbf{do}
\STATE \hspace{0.3cm}$V(t\!+\!1)\!\gets\!\alpha V(t)+(1-\alpha)\!\sum_{l=1}^{L}\!\| \frac{\partial \phi}{\partial\mathbf{W}_{l}}(\mathbf{W}_{1},\!\cdots\!,\mathbf{W}_{L})\|_{F}^{2}$
\STATE \hspace{0.3cm}$\mathbf{A}(t)\gets\dfrac{1}{\sqrt{\frac{V(t+1)}{1-\alpha^{t+1}}}+\varepsilon}\cdot\mathbf{I}$
\STATE \hspace{0.3cm}\textbf{if} $t$ mod $t_{\rm{SPE}}=0$ \textbf{then}
\STATE \hspace{0.7cm}$\mathbf{W}_l(t+1)\gets \mathbf{W}_l(t)-r\mathbf{A}(t)\frac{\partial \phi}{\partial \mathbf{W}_{l}}(\mathbf{W}_{1},\!\cdots\!,\mathbf{W}_{L})$
\STATE \hspace{0.3cm}\textbf{else}
\STATE \hspace{0.7cm}$\mathbf{W}_l(t+1)\gets \mathbf{W}_l(t)-\eta \mathbf{A}(t)\frac{\partial \phi}{\partial \mathbf{W}_{l}}(\mathbf{W}_{1},\!\cdots\!,\mathbf{W}_{L}) $
\STATE \hspace{0.3cm}\textbf{end if}
\STATE \hspace{0.3cm}$\mathbf{W}(t+1)= \mathbf{W}_L(t+1)\mathbf{W}_{L-1}(t+1)\cdots \mathbf{W}_{1}(t+1)$
\STATE \textbf{end for}
\end{algorithmic}
\end{algorithm}

\begin{figure}[!t]
\centering
\includegraphics[width=2.6in]{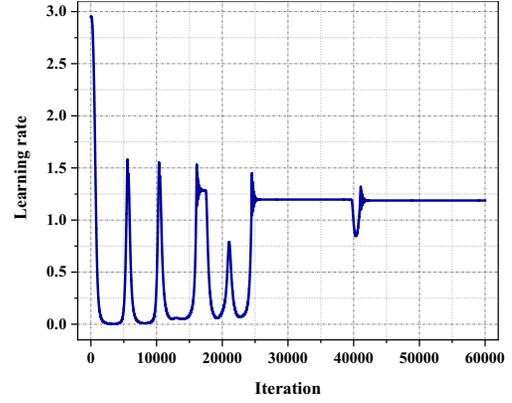}
\caption{Evolution of the learning rate. Note that a 100$\times$100 random matrix with rank-6 is reconstructed by DMF of depth-6, $n_{l}$ = 100 ($l=1,\cdots,6$). The standard deviation of the initialization and learning rate are both $10^{-3}$, and the sampling rate is 30\%.} 
\label{fig_lr}
\end{figure}

\subsection{Landscape Analysis and Convergence Analysis}

To explain the relationship between implicit regularization and the number of SPE stages in DMF, we analyze the second-order convergence property of Algorithm \ref{alg2}.

First, we present Theorem 1 as the main theoretical explanation for Conjecture 1. Then, we state Lemmas 1-3 to prove Theorem 1. The proof details of Lemmas 1-3 and Theorem 1 are provided in the Appendix. 

Before conducting the theoretical analysis, we make some assumptions on the loss function $\ell$ and the range of parameter values $\mathbf{W}_{l}$, for which we provide reasonable explanations below.

\begin{assumption}
$\forall$ $\mathbf{W}, \mathbf{W}^{\prime}\in \mathbb{R}^{n_{L}\times n_{0}}$, a differentiable function $\ell$ is $L_{1}$-gradient Lipschitz if
\begin{equation}
\begin{aligned}
\|\nabla \ell(\mathbf{W})-\nabla \ell(\mathbf{W}^{\prime})\|\leq L_{1}\|\mathbf{W}-\mathbf{W}^{\prime}\|.
\end{aligned}
\end{equation}
\end{assumption}

\begin{assumption}
$\forall$ $\mathbf{W}, \mathbf{W}^{\prime}\in\mathbb{R}^{n_{L}\times n_{0}}$, a twice-differentiable function $\ell$ is $\rho$-Hessian Lipschitz if
\begin{equation}
\begin{aligned}
\|\nabla^2 \ell(\mathbf{W})-\nabla^2 \ell(\mathbf{W}^{\prime})\|\leq \rho\|\mathbf{W}-\mathbf{W}^{\prime}\|.
\end{aligned}
\end{equation}
\end{assumption}

\begin{assumption}
For $1 \leq l \leq L$, $\|\mathbf{W}_{l}(t)\|\leq M$ and $\|\mathbf{W}(t)-\mathbf{W}^{*}\|\leq \|\mathbf{W}^{*}\|_{F}=B$. 
\end{assumption}


Assumptions 1 and 3 are standard in \cite{RN14}. Assumption 1 ensures that the gradient of the function $\ell$ is bounded by the variation of $\mathbf{W}$, which yields a quadratic upper bound on the function $\ell$. Under the setting of approximate balanced initialization and deficiency margin in \cite{RN14}, the boundedness of $\mathbf{W}_{l}(t)$ and $\|\mathbf{W}(t)-\mathbf{W}^{*}\|$ can be obtained in Assumption 3. In particular, $\|\mathbf{W}(t)-\mathbf{W}^{*}\|\leq \|\mathbf{W}^{*}\|_{F}$ can be regarded as the boundedness of $\ell(\mathbf{W})$. These assumption conditions are essential for proving convergence. Assumption 2 shows that the third derivative of $\ell$ exists and is bounded, and $\nabla\ell$ can be approximated as the gradient for the Taylor expansion function of $\ell$ \cite{RN82,RN23,RN26}. Based on these assumptions, we next present our main theorem to demonstrate that the parameter $\mathbf{W}$ in Algorithm \ref{alg2} converges to a second-order critical point. 

\begin{theorem}
Consider Algorithm \ref{alg2} and suppose that the learning rates $\eta$, $r$ meets:
\begin{align}
\label{deqn_ex1}
\eta=\frac{K^{2} \nu^{2} \gamma^{6} \delta^{2}}{64 L^{10} \Lambda^{10}M^{10L-10} B^{4}\rho^{2} \tau^2\omega^2},    
\end{align}
\begin{align}
\label{deqn_ex2}
r=\frac{K\nu\gamma^{4}\delta}{8L^{5}\Lambda^{5}M^{5L-5} B^{2}\rho\tau}.
\end{align}
\noindent{Then, for a small $\tau>0$, with probability $1-\delta$, we reach an $(\tau,\sqrt{\rho\tau})$-critical point in time}
\begin{align}
\label{deqn_ex3}
T = O(R\cdot L^{16}M^{14L-12}B^7\cdot \dfrac{\Lambda^{16}}{\nu^4\delta^4\lambda_{-}^{13}\rho^{\frac{5}{2}}}\cdot \tau^{-5}).
\end{align}
\end{theorem}
\noindent Please note that some variables have been defined in Definitions 1-3.

Theorem 1 states that DMF has the second-order convergence property and its number of iterations is rank-dependent. This theoretically confirms the fact that DMF exhibits the implicit low-rank regularization capability in the loss evolution. To prove Theorem 1, we need to discuss some auxiliary results.

For the analysis in Section 3.1, we define that $\mathbf{W}(t)$ at the SPE stage belongs to $\Omega_t$. Meanwhile, we denote a second-order critical point as $\Omega_t^{c}$ as follows
\begin{equation}
\begin{cases}
\Omega_t =\big\{\mathbf{W}(t)\big|\|\nabla \ell(\mathbf{W}(t))\|^{2}\!\geq\! \tau^{2}\ \text{or}\ \big(\|\nabla \ell(\mathbf{W}(t))\|^{2}\\
\qquad \quad \leq \tau^{2}\ \text{and}\ \lambda_{\min}(\nabla^2 \ell({\mathbf{W}(t)))\leq -\sqrt{\rho}\tau^{\frac{1}{2}}\big)}\big\};\\
\Omega_t^{c} =\big\{\mathbf{W}(t)\big| \|\nabla \ell(\mathbf{W}(t))\|^{2}\leq \tau^{2} \\
 \qquad\qquad \quad  \, \; \text{and}\ \lambda_{\min}(\nabla^2 \ell({\mathbf{W}(t)))\geq -\sqrt{\rho}\tau^{\frac{1}{2}}}\big\}.
\end{cases}
\end{equation}
Accordingly, the convergence process can be divided into three cases: 1) The gradient $\|\nabla \ell(\mathbf{W}(t))\|$ is large (\textbf{Lemma 1}); 2) The gradient $\|\nabla \ell(\mathbf{W}(t))\|$ is small but the minimum eigenvalue of $\nabla^2 \ell(\mathbf{W}(t))$ is less than zero (\textbf{Lemma 2}); 3) the gradient $\|\nabla \ell(\mathbf{W}(t))\|$ is small but the minimum eigenvalue of $\nabla^2 \ell(\mathbf{W}(t))$ is close to zero (\textbf{Lemma 3}). In view of these three cases, the evolution of the function value $\ell$ with the gradient descent dynamics is discussed respectively in below.

\begin{lemma}
Consider a gradient descent step of
\begin{align}
\mathbf{W}(t)=\mathbf{W}_L(t)\mathbf{W}_{L-1}(t)\cdots\mathbf{W}_{1}(t)
\end{align}
on a $L_{1}$-gradient Lipschitz function $\ell$. When the norm of the gradient is large enough $\|\nabla \ell(\mathbf{W}(t ))\|^{2}\geq \tau^{2}$,
\begin{align}
\label{deqn_ex4}
\ell(\mathbf{W}(t+1))-\ell(\mathbf{W}(t))\leq -\frac{1}{2} \sigma_{\min }^{\frac{2L-2}{L}}(\mathbf{W}(t))\eta \tau^{2}.
\end{align}
Suppose that $g_{\rm{thresh}}\leq \frac{1}{2} \sigma_{\min }^{\frac{2L-2}{L}}(\mathbf{W}(t))\eta\tau^{2}$, then it yields the following function decrease
\begin{align}
\label{deqn_lemma1}
\ell(\mathbf{W}(t+1))-\ell(\mathbf{W}(t))\leq -g_{\rm{thresh}}.
\end{align}
\end{lemma}

Lemma 1 guarantees that the function value $\ell$ decreases in each iteration when the gradient is large enough. The recursion in (\ref{deqn_lemma1}) can expect the conclusion that $\mathbf{W}$ converges to a first-order critical point with high probability. However, in the small gradient region, first-order critical points are not sufficient to guarantee convergence to local minima of nonconvex loss landscape. Consequently, we will continue to study the behaviour around a small gradient in Lemma 2.

\begin{lemma}
Suppose that $\| \nabla \ell(\tilde{\mathbf{W}}(t))\|^{2}\leq \tau^{2}$ and the Hessian at $\tilde{\mathbf{W}}(t)$ has a large negative eigenvalue $\lambda_{\min }(\nabla^{2} \ell(\tilde{\mathbf{W}}(t))\leq-\sqrt{\rho} \tau^{-\frac{1}{2}}$. Then, after $k<t_{\rm {thresh}}$ iterations the function value decreases as
\begin{align}
\label{deqn_ex6}
\ell(\mathbf{W}(k+t))-\ell(\tilde{\mathbf{W}}(t)) \leq-\ell_{\rm {thresh}}.
\end{align}
\end{lemma}

In Lemma 2, we mainly focus on the process of escaping saddle points. The strict saddle point $\tilde{\mathbf{W}}$ is defined as satisfying the condition of a small gradient and $\lambda_{\min }(\nabla^{2}\ell(\tilde{\mathbf{W}}(t))\leq-\sqrt{\rho} \tau^{-\frac{1}{2}}$ \cite{RN21}. The theoretical idea here is to use proof by contradiction. Specifically, the upper bound of $\|\mathbf{W}(t)-\tilde{\mathbf{W}}\|_{F}^{2}$ is less than the lower bound, thus it can be concluded that the value of the function $\ell$ can decrease after less than $t_{\rm {thresh}}=O(\tau^{-3})$ iterations instead of a single iteration. 

Lemmas 1 and 2 jointly guarantee the decrease of the function value $\ell$ under $\Omega_t$, so that $\mathbf{W}$ reaches $\Omega_t^{c}$. In the following Lemma 3, we further prove that the function value $\ell$ changes only slightly in $\mathbf{W} \in \Omega_t^{c}$, indicating that $\mathbf{W}$ converges to a second-order critical point.

\begin{lemma}
Suppose that $\|\nabla\ell(\tilde{\mathbf{W}}(t))\|^{2}\leq\tau^{2}$ and that the
absolute value of the minimum eigenvalue of the Hessian at $\tilde{\mathbf{W}}(t)$ is close to zero. Then,
after $k<t_{\rm {thresh }}$ iterations, the function value cannot increase by more than
\begin{equation}
\label{deqn_ex7}
\ell(\mathbf{W}(k+t))-\ell(\tilde{\mathbf{W}}(t)) \leq\frac{\delta \ell_{\rm{thresh}} }{2}.
\end{equation}
\end{lemma}

Based on the above three lemmas, Theorem 1 can be proved.

\begin{figure*}[!t]
\centering
\includegraphics[width=7.1in]{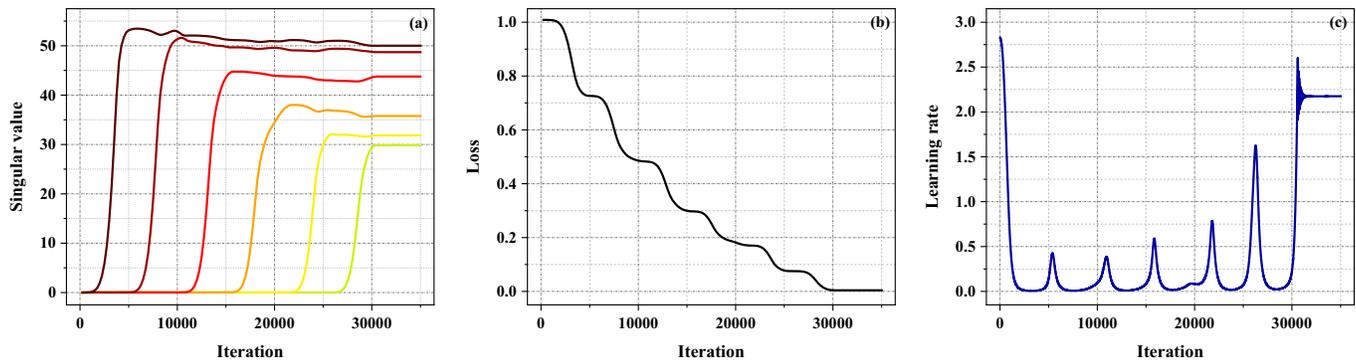}
\caption{The dynamic performance of noisy matrix reconstruction with sampling rate 35\% and Gaussian noise level with SNR = 22 dB. (a) Evolution of singular values; (b) Evolution of the loss; (c) Evolution of the learning rate. Note that a 100$\times$100 ground-truth noise-free matrix with rank-6 is reconstructed by DMF of depth-6, $n_{l}$ = 100 ($l=1,\cdots,6$). The standard deviation of the initialization and learning rate are both $10^{-3}$.}
\label{fig3}
\end{figure*}

\begin{figure*}[!t]
\centering
\includegraphics{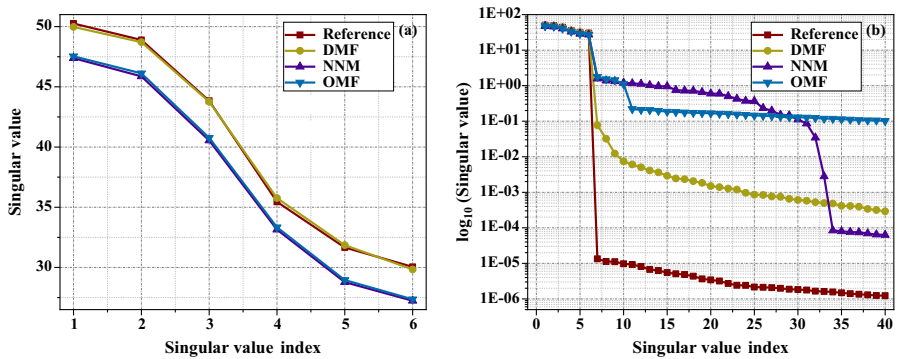}
\caption{The singular values of reconstructed results by different methods. (a) The first six singular values of reference and the reconstructed results by using NNM, OMF and DMF; (b) The logarithmic first forty singular values of reference and the reconstructed results by using DMF, OMF and NNM. Note that a 100$\times$100 noisy synthetic matrix with sampling rate 35\% and Gaussian noise level with SNR = 22 dB is reconstructed by DMF of depth-6. Reference is a noiseless rank-6 random matrix. The compared methods are deep matrix factorization (DMF), the low-rank matrix optimization models based on nuclear norm minimization (NNM) and based on matrix factorization (OMF).}
\label{fig4}
\end{figure*}

\section{Experimental Results}

\subsection{Experiment Setup} 
In this section, we evaluate the performance of DMF on synthetic data to numerically show its implicit low-rank capability. Two optimization model-based low-rank reconstruction methods are compared with DMF, including the nuclear norm minimization (NNM \cite{RN84}) and matrix factorization (OMF \cite{RN95}) Parameters of each method are optimized to obtain the lowest reconstruction error. This error is defined as the relative least normalized error (RLNE) \cite{RN73} 
\begin{align}
\label{deqn_ex}
\text{RLNE}(\mathbf{X},\mathbf{Y})=\frac{\|\mathbf{Y}-\mathbf{X}\|_{F}}{\|\mathbf{X}\|_{F}},
\end{align}
where $\mathbf{X}$ and $\mathbf{Y}$ are the ground-truth noise-free matrix and the reconstructed matrix, respectively.

We generate the noisy synthetic matrix $\mathbf{S}\in\mathbb{R}^{n_{L}\times n_{0}}$ accoding to $\mathbf{S}=\mathbf{X}+\mathbf{N}$ where $\mathbf{N}$ is the added noise matrix. The signal-to-noise ratio (SNR) \cite{RN124} is defined as
\begin{align}
\label{snr}
\text{SNR}(\text{dB})=10\log_{10}{\frac{P_{\mathbf{X}}}{P_{\mathbf{N}}}},
\end{align}
where $P$ is the average power, denoted as $\tfrac{\|\cdot\|_{F}}{n_{L}\times n_{0}}$, $P_{\mathbf{X}}$ and $P_{\mathbf{N}}$ are the average power of the ground-truth noise-free matrix $\mathbf{X}$ and added noise matrix $\mathbf{N}$, respectively. In experiments, we set $\mathbf{X}$ as a noiseless rank-6 random matrix of size 100$\times$100 and add Gaussian random noise matrix $\mathbf{N}$ with SNR = 22 dB to generate a noisy matrix $\mathbf{S}$.

The implement matrix is obtained by randomly removing partial entries in $\mathbf{S}$. As we mentioned before, the sampling rate is defined as the ratio of available entries in the full matrix. To avoid bias, 100 Monte Carlo trials have been tested on multiple sampling rates.

The proposed DMF is implemented in Python 3.6 and Pytorch 1.3.1 as the backend. Both NNM and OMF are performed on MATLAB (Mathworks Inc.). The computational platform is a computer server equipped with one dual Intel Xeon CPUs (2.2 GHz, 24 cores), 128 GB RAM and two Nvidia Tesla K40M GPU cards.

\subsection{Dynamics and Reconstruction Errors}

The evolution process of singular values, the loss and the learning rate is visualized in Fig. \ref{fig3}. The loss function decreases gradually in Fig. \ref{fig3}(a) and exhibits several alternating stages with plateaus and rapid decline. The number of SPE stages is the same as the rank of the ground-truth matrix (Fig. \ref{fig3}(b)). In the plateau stage, the learning rate increase rapidly (Fig. \ref{fig3}(c)). These results imply that DMF has implicit low-rank regularization capability to reconstruct a low-rank matrix from incomplete noisy observations.

Compared with conventional low-rank reconstruction methods, DMF obtains the closest singular values to those of the ground-truth matrix (Fig. \ref{fig4}(a)). All three methods successfully achieve the true rank, i.e. rank = 6. But the log analysis of the singular values in Fig. \ref{fig4}(b) shows that relatively larger values still exist in the $7^{\text{th}}$ and other more singular values. DMF has much smaller incorrect singular values than other methods but the still exists some errors.

Further analysis of the reconstruction errors at different sampling rates is shown in Fig. \ref{RLNE}. DMF always obtains the lowest errors and deviations under all sampling rates. Therefore, DMF is a valuable low-rank reconstruction tool.

\begin{figure}[!t]
\centering
\includegraphics[width=2.6in]{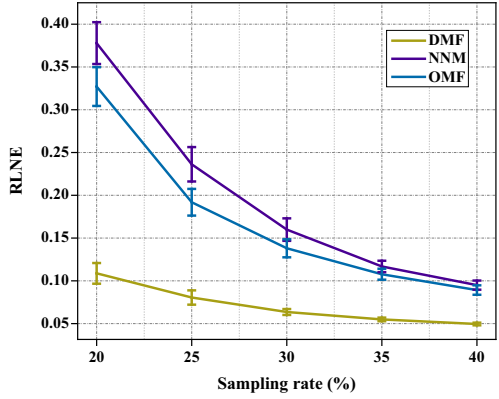}
\caption{Reconstruction error, RLNE, under different values of sampling rate. Note that error bars stand for the standard deviation under 100 Monte Carlo sampling trials with the randomness of sampling pattern.}
\label{RLNE}
\end{figure}



\section{Conclusion}

In summary, through landscape analysis for linear neural networks, we theoretically and experimentally discover the implicit regularization in deep matrix factorization (DMF). First, we find that implicit regularization is exhibited in the loss evolution, that is, the number of saddle point escaping (SPE) stages is equivalent to the rank $R$ of the reconstructed low-rank matrix. Then, we theoretically prove that DMF under discrete gradient dynamics  converges to a second-order critical point after $R$ stages of SPE by landscape analysis. Finally, we experimentally verify the implicit low-rank constraint ability of DMF and shows its lower reconstruction error than compared methods.

For future work, it is worth exploring the effect of saddle points in neural networks and extending the study to curvature. Curvature can obtain richer dynamic information, and regularization terms or optimizers based on the curvature information can help achieve global convergence of neural networks. These concepts could guide the design of deep low-rank matrix factorization networks with better interpretability.


\section*{Acknowledgments}
The authors would like to thank Peng Li, Chunyan Xiong and Zi Wang for helpful discussions and experments.

{\appendix
\section*{Proof of Lemma 1}

Since $\ell(\mathbf{W})$ satisfies $L_{1}$-gradient Lipschitz  as Assumption 1 and $L_{1}=1$, it implies that
\begin{equation}
\begin{aligned}
\ell(\mathbf{W}(i\!+\!1))-&\ell(\mathbf{W}(i))\!\leq\!\left\langle\nabla \ell(\mathbf{W}(i)),\mathbf{W}(i\!+\!1)\!-\!\mathbf{W}(i)\right\rangle\\
& +\frac{1}{2}\|\mathbf{W}(i\!+\!1)-\mathbf{W}(i)\|_{F}^{2}.
\end{aligned}
\end{equation}
The proof idea here is to decompose $\ell(\mathbf{W}(i\!+\!1))-$$\ell(\mathbf{W}(i))$ into two components: 1) $\|\mathbf{W}(i\!+\!1)\!-\!\mathbf{W}(i)\|_{F}^{2}$ and 2) $\left\langle\nabla \ell(\mathbf{W}(i)),\mathbf{W}(i\!+\!1)\!-\!\mathbf{W}(i)\right\rangle$. Then, the upper bounds of two components are proved in Lemmas \rm{\ref{lemmaA1}} and \rm{\ref{lemmaA2}}.

\begin{lemma}
For $\mathbf{W} \in \mathbb{R}^{n_{L}\times n_{0}}$, 
\begin{equation}
\begin{aligned}
\|\mathbf{W}(i\!+\!1)-&\mathbf{W}(i)\|_{F}^{2} \leq \Big(2\eta^{2} L^{2} \Lambda^{2} M^{4L-4}\\
&+4\eta^{4}L^{4} \Lambda^{4} M^{6 L-8} B^{2}\Big)\|\nabla \ell(\mathbf{W}(i))\|_{F}^{2}.
\end{aligned}
\end{equation}
\label{lemmaA1}
\end{lemma}
\begin{proof}
The gradient descent dynamics of $\mathbf{W}$ is as follows
\begin{equation}
\begin{aligned}
&\mathbf{W}(i\!+\!1)-\mathbf{W}(i)\\
&=-\eta \sum_{j=1}^{L} \mathbf{A}(i) \mathbf{W}_{j+1:L}(i)\mathbf{W}_{j+1:L}^{\top}(i) \nabla \ell(\mathbf{W}(i))\mathbf{W}_{1:j-1}^{\top}(i)\\
&\mathbf{W}_{1:j-1}(i)+\mathbf{E}(i).
\end{aligned}
\end{equation}
First, we need the value of  $\|\mathbf{E}(i)\|_{F}$,
\begin{equation}
\begin{aligned}
&\|\mathbf{E}(i)\|_{F} \!\leq\!\eta L \Lambda M^{2 L-2} \|\nabla \ell(\mathbf{W}(i))\|_{F} \sum_{k=2}^{L}\left(\eta L\Lambda M^{L\!-\!2}  B\right)^{k-1}\\
&=\eta^{2} L^{2}\Lambda^{2} M^{3L-4}  B\| \nabla \ell(\mathbf{W}(i)) \|_{F} \sum_{k=2}^{L}\left(\eta L \Lambda M^{L-2} B\right)^{k-2} \\
&\leq 2 \eta^{2} L^{2}\Lambda^{2} M^{3L-4}  B\|\nabla \ell(\mathbf{W}(i))\|_{F},
\end{aligned}
\end{equation}
where the last inequality uses $\eta\leq \dfrac{1}{2L\Lambda M^{L-2} B}$.

Form $\|\mathbf{A}(i)-\hat{\mathbf{A}}(i)\|\leq \Delta$ and $\|\hat{\mathbf{A}}(i)\|\leq\Lambda_{1}$, it holds that $\|\mathbf{A}(i)\| \leq \Lambda_{1}+\Delta = \Lambda$ and $\|\hat{\mathbf{A}}(i)\|\leq\Lambda$, hence we can complete the proof as follow
\begin{align}
&\|\mathbf{W}(i\!+\!1)-\mathbf{W}(i)\|_{F}^{2}\nonumber\\
&\leq 2 \big\|\eta \sum_{j=1}^{L} \mathbf{A}(i) \mathbf{W}_{j+1:L}(i) \mathbf{W}_{j+1:L}^{\top}(i) \nabla \ell(\mathbf{W}(i)) \mathbf{W}_{1:j-1}^{\top}(i)\nonumber\\
&\mathbf{W}_{1:j-1}(i)\big\|_{F}^{2}+2\|\mathbf{E}(i)\|_{F}^{2}\nonumber\\
&\!\leq\!2\eta^{2}\big(\sum_{j=1}^{L}\big\|\mathbf{A}(i) \mathbf{W}_{j+1:L}(i)\mathbf{W}_{j+1:L}^{\top}(i) \nabla \ell(\mathbf{W}(i)) \mathbf{W}_{1:j-1}^{\top}(i)\nonumber\\
&\mathbf{W}_{1:j-1}(i)\big\|_{F}\big)^{2}+2\|\mathbf{E}(i)\|_{F}^{2}\nonumber\\
&\!\leq\!2(\eta^{2} L^{2} \Lambda^{2} M^{4L-4}\!+\!4\eta^{4}L^{4} \Lambda^{4} M^{6 L-8} B^{2})\|\nabla \ell(\mathbf{W}(i))\|_{F}^{2}.
\end{align}
\end{proof}

\begin{corollary}
\label{corollaryA1}
\begin{align}
\|\mathbf{W}(1)\!-\!\tilde{\mathbf{W}}\|_{F}^{2}\!\leq\!2(r^{2} L^{2} \Lambda^{2} M^{4L-4}\!+\!4r^{4}L^{4} \Lambda^{4} M^{6 L-8} B^{2})\tau^{2}.
\end{align}
\end{corollary}

\begin{lemma}
\label{lemmaA2}
For $\mathbf{W} \in \mathbb{R}^{n_{L}\times n_{0}}$, 
\begin{align}
&\left\langle\nabla \ell(\mathbf{W}(i)),\mathbf{W}(i+1)-\mathbf{W}(i)\right\rangle\nonumber\\
&\leq\|\nabla \ell(\mathbf{W}(i))\|_{F}^{2}\Big(\eta L^{3}\Lambda M^{2L-4}\vartheta +2 \eta^{2}L^{2}\Lambda^{2}M^{3L-4}B\\
&-\eta \Lambda\big(\sigma_{\min }^{2}(\mathbf{W}(i))-\frac{3}{2}  L^{2}M^{2L-2}\vartheta\big)^{\frac{L-1}{L}}\Big)\nonumber.
\end{align}
\end{lemma}

\begin{proof}
\begin{align*}
&\left\langle\nabla \ell(\mathbf{W}(i)),\mathbf{W}(i\!+\!1)\!-\!\mathbf{W}(i)\right\rangle\\
&\!\leq\! \Big\langle\nabla \ell(\mathbf{W}(i)),-\eta\sum_{j=1}^{L} \mathbf{A}(i) \mathbf{W}_{j+1:L}(i) \mathbf{W}_{j+1:L}^{\top}(i) \nabla \ell(\mathbf{W}(i)) \\
&\mathbf{W}_{1:j-1}^{\top}(i)\mathbf{W}_{1:j-1}(i)\!+\!\mathbf{E}(i)\Big\rangle\\
&\!\leq\! \Big\langle\nabla \ell(\mathbf{W}(i)),-\eta\sum_{j=1}^{L} \mathbf{A}(i) \mathbf{W}_{j+1:L}(i)\mathbf{W}_{j+1:L}^{\top}(i) \nabla \ell(\mathbf{W}(i))\\ &\mathbf{W}_{1:j-1}^{\top}(i)\mathbf{W}_{1:j-1}(i)\!+\!\eta \Lambda \sum_{j=1}^{L}\big(\mathbf{W}_{L}(i) \mathbf{W}_{L}^{\top}(i)\big)^{L-j}\nabla \ell(\mathbf{W}(i))\\
& \big(\mathbf{W}_{1}^{\top}(i) \mathbf{W}_{1}(i)\big)^{j-1}\Big\rangle+\langle\nabla \ell(\mathbf{W}(i)),\mathbf{E}(i)\rangle\!-\!\Big\langle\nabla \ell(\mathbf{W}(i)),\eta \Lambda\\
& \sum_{j=1}^{L}\big(\mathbf{W}_{L}(i) \mathbf{W}_{L}^{\top}(i)\big)^{L-j}\nabla \ell(\mathbf{W}(i))\left(\mathbf{W}_{1}^{\top}(i) \mathbf{W}_{1}(i)\right)^{j-1}\Big\rangle\\
&\!\leq\!\|\nabla \ell(\mathbf{W}(i))\|_{F}\!\cdot\!\|\eta \sum_{j=1}^{L}\mathbf{A}(i) \mathbf{W}_{j\!+\!1:L}(i)\mathbf{W}_{j\!+\!1:L}^{\top}(i) \nabla \ell(\mathbf{W}(i)) \\
&\mathbf{W}_{1:j-1}^{\top}(i)\mathbf{W}_{1:j-1}(i)\!-\!\eta \Lambda \sum_{j=1}^{L}\big(\mathbf{W}_{L}(i) \mathbf{W}_{L}^{\top}(i)\big)^{L-j} \nabla \ell(\mathbf{W}(i))\\
&\big(\mathbf{W}_{1}^{\top}(i) \mathbf{W}_{1}(i)\big)^{j\!-\!1}\|_{F}\!+\!\|\nabla \ell(\mathbf{W}(i))\|_{F}\|\mathbf{E}(i)\|_{F}\!-\!\Big\langle\nabla \ell(\mathbf{W}(i)), \\
&\eta \Lambda\sum_{j=1}^{L}\big(\mathbf{W}_{L}(i)\mathbf{W}_{L}^{\top}(i)\big)^{L-j}\nabla \ell(\mathbf{W}(i))\big(\mathbf{W}_{1}^{\top}(i) \mathbf{W}_{1}(i)\big)^{j-1}\Big\rangle.
\end{align*}

Then, we need find the upper bounds of the three terms in the last inequality. According to Lemma 5 in \cite{RN14} and $\vartheta$-balanced initialization in Definition 1, we can get the upper bound of the first term
\begin{align*}
&\|\nabla \ell(\mathbf{W}(i))\|_{F}\cdot\Big\|\eta \sum_{j=1}^{L}\mathbf{A}(i) \mathbf{W}_{j+1:L}(i) \mathbf{W}_{j+1:L}^{\top}(i) \nabla \ell(\mathbf{W}(i)) \\
&\mathbf{W}_{1:j-1}^{\top}(i)\mathbf{W}_{1:j-1}(i)\!-\!\eta \Lambda \sum_{j=1}^{L}\big(\mathbf{W}_{L}(i) \mathbf{W}_{L}^{\top}(i)\big)^{L\!-\!j}\nabla \ell(\mathbf{W}(i)) \\
&\big(\mathbf{W}_{1}^{\top}(i) \mathbf{W}_{1}(i)\big)^{j\!-\!1}\Big\|_{F}\\
&\leq\|\nabla \ell(\mathbf{W}(i))\|_{F}\cdot\eta\Lambda\Big[\sum_{j=1}^{L}\Big\|\Big(\mathbf{W}_{j\!+\!1:L}(i) \mathbf{W}_{j\!+\!1:L}^{\top}(i)\!-\!\big(\mathbf{W}_{L}(i)\\ &\mathbf{W}_{L}^{\top}(i)\big)^{L\!-\!j}\Big) \nabla \ell(\mathbf{W}(i)) \mathbf{W}_{1:j\!-\!1}^{\top}(i)\mathbf{W}_{1:j\!-\!1}(i)\Big\|_{F}\!+\!\sum_{j=1}^{L}\Big\|\big(\mathbf{W}_{L}(i)\\ 
& \mathbf{W}_{L}^{\top}(i)\big)^{L-j}\nabla \ell(\mathbf{W}(i))\Big(\mathbf{W}_{1:j-1}^{\top}(i) \mathbf{W}_{1:j-1}(i)-\big(\mathbf{W}_{L}^{\top}(i)\\ &\mathbf{W}_{L}(i)\big)^{L\!-\!j}\Big)\Big\|_{F}\Big]\\
&\leq \|\nabla \ell(\mathbf{W}(i))\|_{F}^{2}\cdot\eta\Lambda\sum_{j=1}^{L}3M^{2L-4}\vartheta(j-1)^2\\
&=\|\nabla \ell(\mathbf{W}(i))\|_{F}^{2}\cdot\eta L^3\Lambda M^{2L-4}\vartheta.
\end{align*}

Next, we calculate the second term
\begin{equation}
\begin{aligned}
&\|\nabla \ell(\mathbf{W}(i))\|_{F}\|\mathbf{E}(i)\|_{F}\\
&\leq\|\nabla \ell(\mathbf{W}(i))\|_{F}^{2}\cdot\eta\Lambda\sum_{j=1}^{L}3M^{2L-4}\vartheta(j-1)^2.
\end{aligned}       
\end{equation}

For the third term, we have 
\begin{align*}
&\!-\!\Big\langle\nabla \ell(\mathbf{W}(i)), \eta \Lambda\sum_{j=1}^{L}\big(\mathbf{W}_{L}(i)\mathbf{W}_{L}^{\top}(i)\big)^{L-j}\nabla \ell(\mathbf{W}(i))\big(\mathbf{W}_{1}^{\top}(i)\\ &\mathbf{W}_{1}(i)\big)^{j-1}\Big\rangle\\
&=-\eta \Lambda \operatorname{vec}(\nabla \ell(\mathbf{W}(i)))^{\top} \cdot \operatorname{vec}\Big(\sum_{j=1}^{L}\big(\mathbf{W}_{L}(i)\mathbf{W}_{L}^{\top}(i)\big)^{L-j}\\
&\nabla \ell(\mathbf{W}(i))\big(\mathbf{W}_{1}^{\top}(i)\mathbf{W}_{1}(i)\big)^{j-1}\Big) \\
&\!=\!-\eta \Lambda \operatorname{vec}(\nabla \ell(\mathbf{W}(i)))^{\top} \! \cdot \! \sum_{j=1}^{L}\Big[\big(\mathbf{W}_{1}^{\top}(i)\mathbf{W}_{1}(i)\big)^{j-1}\otimes\big(\mathbf{W}_{L}(i)\\
&\mathbf{W}_{L}^{\top}(i)\big)^{L-j}\Big]\cdot\operatorname{vec}(\nabla \ell(\mathbf{W}(i))) \\
&\!=\!-\eta \Lambda \operatorname{vec}(\nabla \ell(\mathbf{W}(i)))^{\top}\! \cdot\!\sum_{j=1}^{L}\Big[\big(\mathbf{U}\mathbf{D}^{j-1} \mathbf{U}^{\top}\big) \!\otimes\!\big(\mathbf{V} \mathbf{E}^{L-j} \mathbf{V}^{\top}\big)\Big]\\
&\cdot \operatorname{vec}(\nabla \ell(\mathbf{W}(i))) \\
&\!=\!-\eta \Lambda \operatorname{vec}(\nabla \ell(\mathbf{W}(i)))^{\top}\! \cdot \!\big(\mathbf{U} \otimes \mathbf{V}\big) \sum_{j=1}^{L}\big(\mathbf{D}^{j-1} \otimes \mathbf{E}^{L-j}\big)\big(\mathbf{U}^{\top} \\
&\otimes \mathbf{V}^{\top}\big) \cdot \operatorname{vec}(\nabla \ell(\mathbf{W}(i))) \\
&=-\eta \Lambda \operatorname{vec}(\nabla \ell(\mathbf{W}(i)))^{\top}\cdot \mathbf{O} \mathbf{F} \mathbf{O}^{\top} \cdot \operatorname{vec}(\nabla \ell(\mathbf{W}(i)))\\
&\leq-\|\nabla \ell(\mathbf{W}(i))\|_{F}^{2}\cdot\eta \Lambda \Big(\sigma_{\min }^{2}(\mathbf{W}(i))-\frac{3}{2} L^{2} M^{2L-2}\vartheta \Big)^{\frac{L-1}{L}},
\end{align*}
where eigenvalue decompositions $\mathbf{W}_{1}^{\top}(i)\mathbf{W}_{1}(i)=\mathbf{U}\mathbf{D} \mathbf{U}^{\top}$ and $\mathbf{W}_{L}(i)\mathbf{W}_{L}^{\top}(i)=\mathbf{V} \mathbf{E} \mathbf{V}^{\top} $, $\mathbf{O}=\mathbf{U} \otimes \mathbf{V}$.  In the last inequality use Lemma 5 in \cite{RN14}, 
\begin{equation}
\begin{aligned}
\lambda_{\rm{min}}{(\mathbf{F})}&=\sum_{j=1}^{L}\lambda_{\rm{min}}\big(\mathbf{D}^{j-1} \otimes \mathbf{E}^{L-j}\big)\\
&\ge \lambda_{\rm{min}}\big(\mathbf{D}^{L-1} \otimes\mathbf{I}\big)+\lambda_{\rm{min}}\big(\mathbf{I}\otimes
\mathbf{E}^{L-1}\big)\\
&\ge \Big(\sigma_{\min }^{2}(\mathbf{W}(i))-\frac{3}{2} L^{2} M^{2L-2}\vartheta \Big)^{\frac{L-1}{L}}.
\end{aligned}    
\end{equation}

\end{proof}

\begin{lemma}[\emph {Restate of Lemma 1}]
\label{lemmaA3}
 When the norm of the gradient is large enough $\|\nabla \ell\left(\mathbf{W}\left(i\right)\right)\|^{2}\geq \tau^{2}$, it establishes that
\begin{equation}
\begin{aligned}
\ell(\mathbf{W}(i+1))-\ell(\mathbf{W}(i))\leq -\frac{1}{2}\sigma_{\min}^{\frac{2L-2}{L}}(\mathbf{W}(i))\eta\|\nabla \ell(\mathbf{W}(i))\|_{F}^{2}.
\end{aligned}
\end{equation}
Suppose that $g_{\rm{thresh}}\leq \frac{1}{2} \sigma_{\min }^{\frac{2L-2}{L}}(\mathbf{W}(i))\eta\tau^{2}$, then yields the following function decrease,
\begin{align}
\ell\left(\mathbf{W}\left(i+1\right)\right)-\ell\left(\mathbf{W}\left(i\right)\right)\leq -g_{\rm{thresh}}.
\end{align}
\end{lemma}
\begin{proof}
According the Lemmas \rm{\ref{lemmaA1}} and \rm{\ref{lemmaA2}}, we get
\begin{align*}
&\ell(\mathbf{W}(i\!+\!1))\!-\!\ell(\mathbf{W}(i))\\
&\leq\|\nabla \ell(\mathbf{W}(i))\|_{F}^{2}\Big(\eta L^{3}\Lambda M^{2L-4}\vartheta +2 \eta^{2}L^{2}\Lambda^{2}M^{3L-4}B\\
&-\eta \Lambda\big(\sigma_{\min }^{2}(\mathbf{W}(i))-\frac{3}{2}  L^{2}M^{2L-2}\vartheta\big)^{\frac{L-1}{L}}+2\eta^{2} L^{2} \Lambda^{2} M^{4L-4}\\
&+4\eta^{4}L^{4} \Lambda^{4} M^{6 L-8} B^{2}\Big).
\end{align*}

By setting $\eta$ and $\vartheta$ as follows

\begin{equation}
\begin{aligned}
\eta \leq \min &\left\{  \frac{1}{2L\Lambda M^{L-2}B}, \frac{{\sigma_{\min }^{\frac{2(L-1)}{L}}(\mathbf{W}(i))}}{2L^2\Lambda^2 M^{3L-4}B},\right.\\
&\left.\frac{{\sigma_{\min }^{\frac{2(L-1)}{L}}(\mathbf{W}(i))}}{24L^2\Lambda^2 M^{4L-4}B}, \frac{{\sigma_{\min }^{\frac{2(L-1)}{3L}}(\mathbf{W}(i))}}{{\big(96L^4\Lambda^4 M^{6L-8}B^2\big)}^{\frac{1}{3}}} \right\},
\end{aligned}
\end{equation}

\begin{equation}
\begin{aligned}
\vartheta\leq \min \left\{ \frac{{\sigma_{\min }^{\frac{2(L-1)}{L}}(\mathbf{W}(i))}}{8L^3 \Lambda M^{2L-4}} , \frac{\sigma_{\min }^2(\mathbf{W}(i))}{6L^2\Lambda M^{2L-2}} \right\},
\end{aligned}
\end{equation}
it holds that
\begin{equation}
\begin{aligned}
\ell(\mathbf{W}(i+1))\!-\!\ell(\mathbf{W}(i))\leq -\frac{1}{2}\sigma_{\min}^{\frac{2L-2}{L}}(\mathbf{W}(i))\eta\|\nabla \ell(\mathbf{W}(i))\|_{F}^{2}.
\end{aligned}
\end{equation}
Form $\|\nabla\ell\left(\mathbf{W}\left(i\right)\right)\|^{2}\geq \tau^{2}$ and $g_{\rm{thresh}}\leq \frac{1}{2} \sigma_{\min }^{\frac{2L\!-\!2}{L}}(\mathbf{W}(i))\eta\tau^{2}$, the proof result is finally completed as
\begin{align}
\ell\left(\mathbf{W}\left(i+1\right)\right)-\ell\left(\mathbf{W}\left(i\right)\right)\leq -g_{\rm{thresh}}.
\end{align}

\end{proof}

\begin{corollary}
\label{corollaryA2}
\begin{equation}
\begin{aligned}
\ell(\mathbf{W}(1))-\ell(\tilde{\mathbf{W}})\leq 8r^{2} L^{4} \Lambda^{4} M^{6L-8}B^{2}\tau^{2}.
\end{aligned}
\end{equation}
\end{corollary}

\section*{Proof of Lemma 2}
Here, we use the proof by contradiction as done in \cite{RN37,RN38}. First, the upper bound of $\|\mathbf{W}(t)-\tilde{\mathbf{W}}\|_{F}^{2}$ is obtained under the hypothetical conclusion. Then, we deduce the lower bound of $\|\mathbf{W}(t)-\tilde{\mathbf{W}}\|_{F}^{2}$. Under sufficient large number of iterations and certain parameter settings, the lower bound can be greater than the upper bound, so as to obtain desired conclusion. Our proof is divided into three parts as shown below.

\noindent{\textbf{Part 1: Upper bounding the distance on the iterates in terms of function decrease.}}

When $\mathbf{W}(t)$ is close to $\tilde{\mathbf{W}}$, we assume that the reduction of the function  cannot obtain the desired result in the  $t<t_{\rm {thresh }}$ iterations :
\begin{equation}
\label{Eq1}
 \ell(\mathbf{W}(t))-\ell(\tilde{\mathbf{W}}) \geq -l_{\text {thresh }}.
\end{equation}

\begin{lemma}
\label{lemmaA4}
Assume that (\ref{Eq1}) holds,
\begin{equation}
\begin{aligned}
\|\mathbf{W}(t)-\tilde{\mathbf{W}}\|_{F}^{2}\leq t\eta\sigma_{\min}^{\frac{2-2L}{L}}H\cdot\big(l_{\rm {thresh}}+Q)+4Q,
\end{aligned}
\end{equation}
where $H=32L^{4} \Lambda^{4} M^{6L-8}B^{2}$, $Q=8r^{2} L^{4} \Lambda^{4} M^{6L-8}B^{2}\tau^{2}$.
\end{lemma}
\begin{proof}
Using Lemma \rm{\ref{lemmaA3}} and Corollary \rm{\ref{corollaryA2}}, we have 
\begin{equation}
\begin{aligned}
&-l_{\text {thresh }} 
\leq \ell(\mathbf{W}(t))-\ell(\tilde{\mathbf{W}}) \\
&=\sum_{i=1}^{t-1}\left[\ell(\mathbf{W}(i+1))-\ell(\mathbf{W}(i))\right]+\ell(\mathbf{W}(1))-\ell(\tilde{\mathbf{W}})\nonumber\\
&\leq-\sum_{i=1}^{t-1}  \frac{1}{2}\sigma_{\min}^{\frac{2L-2}{L}}\eta\|\nabla \ell(\mathbf{W}(i))\|_{F}^{2}+8r^{2} L^{4} \Lambda^{4} M^{6L-8}B^{2}\tau^{2}.
\end{aligned}
\end{equation}
By rearranging, we can get a bound on the gradient norms:
\begin{equation}
\label{Eq2}
\begin{aligned}
\sum_{i=1}^{t-1}\|\nabla \ell(\mathbf{W}(i))\|_{F}^{2} \leq2&\sigma_{\min}^{\frac{2-2L}{L}}\eta^{-1}\big(\ell_{\text {thresh}}\\
&+8r^{2} L^{4} \Lambda^{4} M^{6L-8}B^{2}\tau^{2}\big).
\end{aligned}
\end{equation}

According to Lemma \rm{\ref{lemmaA1}} and Corollary \rm{\ref{corollaryA1}}, we can get:
\begin{align*}
&\|\mathbf{W}(t)-\tilde{\mathbf{W}}\|_{F}^{2} \\
&\leq 2\|\mathbf{W}(t)-\mathbf{W}(1)\|_{F}^{2}+2\|\mathbf{W}(1)-\tilde{\mathbf{W}}\|_{F}^{2}\\
&\leq 2\left(\left\|\sum_{i=1}^{t-1} \mathbf{W}(i+1)-\mathbf{W}(i)\right\|_{F}\right)^{2}+2\|\mathbf{W}(1)-\tilde{\mathbf{W}}\|_{F}^{2}\\
&\leq 2 t \sum_{i=1}^{t-1}\|\mathbf{W}(i+1)-\mathbf{W}(i)\|_{F}^{2}+2\|\mathbf{W}(1)-\tilde{\mathbf{W}}\|_{F}^{2}\\
&\leq 2 t\cdot\sum_{i=1}^{t-1}\Big[\big(2\eta^{2} L^{2} \Lambda^{2} M^{4L-4}+4\eta^{4}L^{4} \Lambda^{4} M^{6 L-8} B^{2}\big)\\
&\|\nabla \ell(\mathbf{W}(i))\|_{F}^{2}\Big] \!+\!4(r^{2} L^{2} \Lambda^{2} M^{4L-4}\!+\!4r^{4}L^{4} \Lambda^{4} M^{6 L-8} B^{2})\tau^{2}\\
&\leq16t\eta^2 L^{4} \Lambda^{4} M^{6L-8}B^{2}\cdot\sum_{i=1}^{t-1}\|\nabla \ell(\mathbf{W}(i))\|_{F}^{2}\\
&+32r^{2} L^{4} \Lambda^{4} M^{6L-8}B^{2}\tau^{2}.\\
\end{align*}

To replace (\ref{Eq2}) into the above inequality, we can complete the proof
\begin{equation}
\begin{aligned}
&\|\mathbf{W}(t)-\tilde{\mathbf{W}}\|_{F}^{2}\\
&\!\leq\!32\sigma_{\min}^{\frac{2-2L}{L}}t\eta L^{4} \Lambda^{4} M^{6L-8}B^{2}
\big(l_{\text {thresh }}\!+\!8r^{2} L^{4} \Lambda^{4} M^{6L-8}B^{2}\tau^{2}\big)\\
&+32r^{2} L^{4} \Lambda^{4} M^{6L-8}B^{2}\tau^{2}\\
&=t\eta\sigma_{\min}^{\frac{2-2L}{L}}H\cdot\big(\ell_{\text {thresh}}+Q)+4Q,
\end{aligned}
\end{equation}
where $H=32L^{4} \Lambda^{4} M^{6L-8}B^{2}$, $Q=8r^{2} L^{4} \Lambda^{4} M^{6L-8}B^{2}\tau^{2}$.
\end{proof}

\noindent{\textbf{Part 2: Quadratic approximation.}}

At the point $\tilde{\mathbf{W}}$, we can use a second order Taylor expansion approximation of the function $\ell$ \cite{RN23}:
\begin{equation}
\begin{aligned}
g(\mathbf{W})=\ell(\tilde{\mathbf{W}})+(\mathbf{W}&-\tilde{\mathbf{W}})^{\top}\nabla \ell(\tilde{\mathbf{W}})\\
&+\frac{1}{2}(\mathbf{W}-\tilde{\mathbf{W}})^{\top}\mathcal{H}(\mathbf{W}-\tilde{\mathbf{W}}).
\end{aligned}
\end{equation}

\begin{lemma}[\emph{Nesterov, 2013 \cite{RN82}}]
 For every twice differentiable $\rho$-Hessian Lipschitz function $\ell$ we have
\begin{equation}
\begin{aligned}
\|\nabla \ell(\mathbf{W})-\nabla g(\mathbf{W})\| \leq \frac{\rho}{2}\|\mathbf{W}-\tilde{\mathbf{W}}\|^{2}.
\end{aligned}
\end{equation}
\end{lemma}
 
Through the above discussion, we further transform the gradient update of $\mathbf{W}(t+1)$ into
\begin{align*}
&\mathbf{W}(t+1)-\tilde{\mathbf{W}}=\mathbf{W}(t)-\tilde{\mathbf{W}}-\eta \mathbf{A}(t)\nabla \mathbf{L}(t)+\mathbf{E}(t)\\
&=\mathbf{W}(t)-\tilde{\mathbf{W}}-\eta \mathbf{A}(t)\nabla \ell(\tilde{\mathbf{W}})-\eta \mathbf{A}(t)\mathcal{H}(\mathbf{W}(t)-\tilde{\mathbf{W}})\\
&+\eta \mathbf{A}(t)\Big[\nabla g(\mathbf{W}(t))-\nabla \mathbf{L}(t)\Big]+\mathbf{E}(t)\\
&=\big(\mathbf{I}-\eta \mathbf{A}(t)\mathcal{H}\big)(\mathbf{W}(t)-\tilde{\mathbf{W}})\!+\!\eta \mathbf{A}(t)\Big(\nabla g(\mathbf{W}(t))\!-\!\nabla \ell(\tilde{\mathbf{W}})\\
&-\nabla \mathbf{L}(t)\Big)+\mathbf{E}(t)\\
&=\big(\mathbf{I}-\eta \hat {\mathbf{A}}\mathcal{H}\big)(\mathbf{W}(t)-\tilde{\mathbf{W}})\!+\!\eta \hat {\mathbf{A}}\Big(\nabla g(\mathbf{W}(t))\!-\!\nabla \ell(\tilde{\mathbf{W}})\\
&-\nabla \mathbf{L}(t)\Big)+\eta\big(\hat {\mathbf{A}}-\mathbf{A}(t)\big)\mathcal{H}(\mathbf{W}(t)-\tilde{\mathbf{W}})+\eta\big(\mathbf{A}(t)-\hat {\mathbf{A}}\big)\\
&\Big(\nabla g(\mathbf{W}(t))\!-\!\nabla \ell(\tilde{\mathbf{W}})-\nabla \mathbf{L}(t)\Big)+\mathbf{E}(t)\\
&=\big(\mathbf{I}-\eta \hat {\mathbf{A}}\mathcal{H}\big)(\mathbf{W}(t)-\tilde{\mathbf{W}})\!+\!\eta \hat {\mathbf{A}}\Big(\nabla g(\mathbf{W}(t))\!-\!\nabla \ell(\tilde{\mathbf{W}})\\
&\!-\!\nabla \mathbf{L}(t)\Big)\!+\!\mathbf{E}(t)\!+\!\eta\big(\hat {\mathbf{A}}\!-\!\mathbf{A}(t)\big)\mathcal{H}(\mathbf{W}(t)\!-\!\tilde{\mathbf{W}})\!+\!\eta\big(\mathbf{A}(t)\!-\!\hat {\mathbf{A}}\big)\\
&\Big(\mathcal{H}(\mathbf{W}(t)-\tilde{\mathbf{W}})-\nabla \mathbf{L}(t)\Big)\\
&=\big(\mathbf{I}-\eta \hat {\mathbf{A}}\mathcal{H}\big)(\mathbf{W}(t)-\tilde{\mathbf{W}})\!+\!\eta \hat {\mathbf{A}}\Big(\nabla g(\mathbf{W}(t))\!-\!\nabla \ell(\tilde{\mathbf{W}})\\
&-\nabla \mathbf{L}(t)\Big)+\eta\big(\hat {\mathbf{A}}-\mathbf{A}(t)\big)\nabla \mathbf{L}(t)+\mathbf{E}(t)\\
&=\big(\mathbf{I}-\eta \hat {\mathbf{A}}\mathcal{H}\big)(\mathbf{W}(t)-\tilde{\mathbf{W}})\!+\!\eta \hat {\mathbf{A}}\Big(\nabla g(\mathbf{W}(t))\!-\!\nabla \ell(\tilde{\mathbf{W}})\\
&-\nabla \mathbf{L}(t)\Big)+\eta\big(\hat {\mathbf{A}}-\hat {\mathbf{A}}(t)\big)\nabla \mathbf{L}(t)+\eta\big(\hat {\mathbf{A}}(t)-\mathbf{A}(t)
\big)\nabla \mathbf{L}(t)\\
&+\mathbf{E}(t).\\
\end{align*}

Next, we rearrange the last inequality into recursive form
\begin{equation}
\label{Eq3}
\begin{aligned}
\mathbf{W}&(t+1)-\tilde{\mathbf{W}}\\
&=\bm{u}(t)+\eta\big(\bm{\delta}(t)+\mathbf{d}(t)+\bm{\zeta}(t)+\bm{\chi}(t)+\bm{\iota}(t)\big).
\end{aligned}
\end{equation}

We define the form of parameters in (\ref{Eq3}),
\begin{equation*}
\begin{aligned}
&\bm{u}(t)=-\big(\mathbf{I}-\eta \hat {\mathbf{A}}\mathcal{H}\big)^t(r\mathbf{A}\nabla \tilde{f}),\\
&\bm{\delta}(t)=\sum_{i=1}^{t}(\mathbf{I}-\eta \hat {\mathbf{A}}\mathcal{H}\big)^{t-i}\hat {\mathbf{A}}\big(\nabla g(\mathbf{W}(i))-\nabla \mathbf{L}(i)\big),\\
&\mathbf{d}(t)=-\sum_{i=1}^{t}(\mathbf{I}-\eta \hat {\mathbf{A}}\mathcal{H}\big)^{t-i}\hat {\mathbf{A}}\nabla \ell( \tilde{\mathbf{W}}),\\
&\bm{\zeta}(t)=\sum_{i=1}^{t}\eta^{-1}(\mathbf{I}-\eta \hat {\mathbf{A}}\mathcal{H}\big)^{t-i}\mathbf{E}(i)+\eta^{-1}\tilde{\mathbf{E}},\\
&\bm{\chi}(t)=\sum_{i=1}^{t}(\mathbf{I}-\eta \hat {\mathbf{A}}\mathcal{H}\big)^{t-i}\big(\hat {\mathbf{A}}-\hat {\mathbf{A}}(i)\big)\nabla \mathbf{L}(i),\\
&\bm{\iota}(t)=\sum_{i=1}^{t}(\mathbf{I}-\eta \hat {\mathbf{A}}\mathcal{H}\big)^{t-i}\big(\hat {\mathbf{A}}(i)-\mathbf{A}(i)
\big)\nabla \mathbf{L}(i).\\
\end{aligned}
\end{equation*}

\noindent{\textbf{Part 3: Lower bounding the iterate distance.}}

We now find a lower bound on the $\|\mathbf{W}(t+1)\!-\!\tilde{\mathbf{W}}\|$ derived in the previous part. This conclusion will contradict the Lemma \rm{\ref{lemmaA4}}.

\begin{lemma}
\begin{equation}
\begin{aligned}
\langle \bm{u}(t),\mathbf{d}(t)\rangle \geq 0.
\end{aligned}
\end{equation}
\end{lemma}
\begin{proof}
\begin{align*}
&\langle \bm{u}(t),\mathbf{d}(t)\rangle
=\mathrm{tr}\big(\bm{u}^{\top}(t)\mathbf{d}(t)\big)\\
&=\mathrm{tr}\big( r\nabla\tilde{f}^{\top} \hat {\mathbf{A}}^{\top}  \sum_{i=1}^{t}\big(\mathbf{I}-\eta \hat {\mathbf{A}}\mathcal{H}\big)^{2t-i} \hat {\mathbf{A}}\nabla \ell(\tilde{\mathbf{W}}) \big)\\
&=r\mathrm{tr} \big[ \big(\sum_{j=1}^{L}\tilde{\mathbf{W}}_{j+1:L}\tilde{\mathbf{W}}_{j+1:L}^{\top}\nabla \ell(\tilde{\mathbf{W}})\tilde{\mathbf{W}}_{1:j-1}^{\top}\tilde{\mathbf{W}}_{1:j-1}\big)\\
&\cdot \hat {\mathbf{A}}^{\top}  \sum_{i=1}^{t}\big(\mathbf{I}-\eta \hat {\mathbf{A}}\mathcal{H}\big)^{2t-i}  \hat {\mathbf{A}} \nabla \ell(\tilde{\mathbf{W}})\big]\\
&=r\cdot \operatorname{vec}\big(\sum_{j=1}^{L}\tilde{\mathbf{W}}_{j+1:L}\tilde{\mathbf{W}}_{j+1:L}^{\top}\nabla \ell(\tilde{\mathbf{W}})\tilde{\mathbf{W}}_{1:j-1}^{\top}\tilde{\mathbf{W}}_{1:j-1}\big)\\
&\cdot\operatorname{vec}\big( \hat {\mathbf{A}}^{\top}  \sum_{i=1}^{t}\big(\mathbf{I}-\eta \hat {\mathbf{A}}\mathcal{H}\big)^{2t-i}  \hat {\mathbf{A}} \nabla \ell(\tilde{\mathbf{W}}) \big)\\
&=r\!\cdot\! \operatorname{vec}\big(\nabla \ell(\tilde{\mathbf{W}})\big)^{\top}\sum_{j=1}^{L} \big[ \big(\tilde{\mathbf{W}}_{j+1:L}\tilde{\mathbf{W}}_{j+1:L}^{\top}\big) \otimes \big( \tilde{\mathbf{W}}_{1:j-1}^{\top}\\
&\tilde{\mathbf{W}}_{1:j-1} \big)\big]\!\cdot\! \big[ \mathbf{I} \otimes \big(  \hat {\mathbf{A}}^{\top}\sum_{i=1}^{t}\big(\mathbf{I}-\eta \hat {\mathbf{A}}\mathcal{H}\big)^{2t-i} \hat {\mathbf{A}} \big)\big]\operatorname{vec}\big(\nabla \ell(\tilde{\mathbf{W}})\big)\\
&=rL\!\cdot\! \sum_{i=1}^{t}\operatorname{vec}\big(\nabla \ell(\tilde{\mathbf{W}})\big)^{\top} \big[ \big(\tilde{\mathbf{U}}\tilde{\mathbf{D}}^2\tilde{\mathbf{U}}^{\top}\big) \otimes \big( \tilde{\mathbf{N}}\tilde{\mathbf{E}}^2\tilde{\mathbf{N}}^{\top} \big)\big]\\
&\cdot\big[\big(\mathbf{I}\otimes \hat {\mathbf{A}}^{\top}\big)\big(\mathbf{I}\otimes \big(\mathbf{I}-\eta \hat {\mathbf{A}}\mathcal{H}\big)^{2t-i} \big) (\mathbf{I}\otimes \hat {\mathbf{A}}\big)\big]\operatorname{vec}\big(\nabla \ell(\tilde{\mathbf{W}})\big)\\
&=rL\cdot\operatorname{vec}\big(\nabla \ell(\tilde{\mathbf{W}})\big)^{\top} \big(  \tilde{\mathbf{U}}\otimes\tilde{\mathbf{N}}\big) \big(  \tilde{\mathbf{D}}^2\otimes\tilde{\mathbf{E}}^2\big) \big(  \tilde{\mathbf{U}}\otimes\tilde{\mathbf{N}}\big)^{\top} \\
&\big(\mathbf{I}\otimes \hat {\mathbf{A}}^{\top}\big)\big(\mathbf{I}\otimes \big(\mathbf{I}-\eta \hat {\mathbf{A}}\mathcal{H}\big)^{2t-i} \big) \big(\mathbf{I}\otimes \hat {\mathbf{A}}\big)\operatorname{vec}\big(\nabla \ell(\tilde{\mathbf{W}})\big)\geq 0,
\end{align*}
where singular value decompositions $\tilde{\mathbf{W}}_{j+1:L}\!=\!\tilde{\mathbf{U}}\tilde{\mathbf{D}}\tilde{\mathbf{V}}^{\top}$ and $\tilde{\mathbf{W}}_{1:j-1}\!=\!\tilde{\mathbf{M}}\tilde{\mathbf{E}}\tilde{\mathbf{N}}^{\top}$. Setting  enough small $\eta$, we have $\|\mathbf{I}-\eta \mathbf{A}\mathcal{H} \|\leq1$.
\end{proof}

\begin{lemma}
\begin{equation}
\begin{aligned}
\|\bm{u}(t)\|_{F}^2\geq \kappa^{2t}r^2\nu.
\end{aligned}
\end{equation}
\end{lemma}
\begin{proof}
Since $\hat {\mathbf{A}}$ is a diagonally positive definite and $\mathcal{H}$ is symmetric matrix, 
\begin{equation}
\label{Eq4}
\begin{aligned}
\hat {\mathbf{A}}^{\frac{1}{2}}\mathcal{H}\hat {\mathbf{A}}^{\frac{1}{2}}&=\mathbf{J}\mathbf{M}\mathbf{J}^{\top}\\
\hat {\mathbf{A}}\mathcal{H}&=\big(\hat {\mathbf{A}}^{\frac{1}{2}}\mathbf{J}\big)\mathbf{M}\big(\hat {\mathbf{A}}^{\frac{1}{2}}\mathbf{J}\big)^{-1}
\end{aligned}
\end{equation}
where $\mathbf{M}$ is the eigenvalue matrix of $\hat {\mathbf{A}}^{\frac{1}{2}}\mathcal{H}\hat {\mathbf{A}}^{\frac{1}{2}}$ and $\gamma=\big\vert \lambda_{\min}(\hat {\mathbf{A}}^{\frac{1}{2}}\mathcal{H}\hat {\mathbf{A}}^{\frac{1}{2}}) \big\vert$. With (\ref{Eq4}), we have 

\begin{equation}
\label{Eq5}
\begin{aligned}
\big(\mathbf{I}-\eta \hat {\mathbf{A}}\mathcal{H}\big)^t&=\Big( \mathbf{I}-\eta\big(\hat {\mathbf{A}}^{\frac{1}{2}}\mathbf{J}\big)\mathbf{M}\big(\hat {\mathbf{A}}^{\frac{1}{2}}\mathbf{J}\big)^{-1} \Big)^{t}\\
&=\big(\hat {\mathbf{A}}^{\frac{1}{2}}\mathbf{J}\big)\big(\mathbf{I}-\eta\mathbf{M}\big)^{t}\big(\hat {\mathbf{A}}^{\frac{1}{2}}\mathbf{J}\big)^{-1}.\\
\end{aligned}
\end{equation}

Setting the unit vector $\mathbf{e}=c\big(\hat {\mathbf{A}}^{-\frac{1}{2}}\mathbf{J}\big)\mathbf{e}_1$ and  $\gamma=\big\vert \lambda_{\min}(\hat {\mathbf{A}}^{\frac{1}{2}}\mathcal{H}\hat {\mathbf{A}}^{\frac{1}{2}}) \big\vert$,
\begin{equation}
\begin{aligned}
\mathbf{e}^{\top}\big(\mathbf{I}-\eta\hat {\mathbf{A}}\mathcal{H}\big)^t&=c\mathbf{e}_1^{\top}\big(\mathbf{I}-\eta\mathbf{M}\big)^{t}\big(\hat {\mathbf{A}}^{\frac{1}{2}}\mathbf{J}\big)^{-1}\\
&=c(1+\eta\gamma)^{t}\mathbf{e}^{\top},
\end{aligned}
\end{equation}
where $\mathbf{e}_1$ is the first standard basis vector and $c$ is a scalar constant.

Finally, we use matrix norm consistent to relax $\|\bm{u}(t)\|_{F}^2$,
\begin{equation}
\begin{aligned}
\|\bm{u}(t)\|_{F}^2&=\|\mathbf{e}\|^2\|\bm{u}(t)\|_{F}^2\\
&\geq \|\mathbf{e}^{\top}\big(\mathbf{I}-\eta \hat {\mathbf{A}}\mathcal{H}\big)^t(-r\hat {\mathbf{A}}\nabla \tilde{f})\|^2\\
&=(1+\eta\gamma)^{2t}r^2\| \mathbf{e}^{\top}\hat {\mathbf{A}}\nabla \tilde{f}\|^2\\
&=(1+\eta\gamma)^{2t}r^2\cdot\big(\mathbf{e}^{\top}\hat {\mathbf{A}}\nabla\tilde{f}\nabla\tilde{f}^{\top}\hat {\mathbf{A}}^{\top}\mathbf{e}\big)\\
&=(1+\eta\gamma)^{2t}r^2\cdot\lambda_{\min}(\hat {\mathbf{A}}\nabla\tilde{f}\nabla\tilde{f}^{\top}\hat {\mathbf{A}}^{\top})\\
&=\kappa^{2t}r^2\nu,
\end{aligned}
\end{equation}
where $\nu=\lambda_{\min}(\hat {\mathbf{A}}\nabla\tilde{f}\nabla\tilde{f}^{\top}\hat {\mathbf{A}}^{\top})$ in Definition 3 and $\kappa=1+\eta\gamma$.
\end{proof}

\begin{lemma}
\begin{equation}
\begin{aligned}
\|\bm{u}(t)\|_F\leq \kappa^{t}r L\Lambda M^{2L-2}\tau.
\end{aligned}
\end{equation}
\end{lemma}
\begin{proof}
Combining (\ref{Eq5}), we have
\begin{equation}
\begin{aligned}
&\|\mathbf{I}-\eta\hat {\mathbf{A}}\mathcal{H}\|\\
&= \|\mathbf{I}-\eta\big(\hat{\mathbf{A}}^{\frac{1}{2}}\mathbf{J}\big)\mathbf{M}\big(\hat {\mathbf{A}}^{\frac{1}{2}}\mathbf{J}\big)^{-1} \|\\
&= \|\big(\hat {\mathbf{A}}^{\frac{1}{2}}\mathbf{J}\big)\big(\mathbf{I}-\eta\mathbf{M}\big)\big(\hat {\mathbf{A}}^{\frac{1}{2}}\mathbf{J}\big)^{-1} \|\\
&\leq \|\mathbf{I}-\eta\mathbf{M} \| \\
&= 1+\eta\gamma.
\end{aligned}
\end{equation}
Hence, the proof is completed as follows
\begin{align*}
\|\bm{u}(t)\|_F &
\leq \|\big(\mathbf{I}-\eta\hat {\mathbf{A}}\mathcal{H}\big)^t(-r\hat {\mathbf{A}}\nabla \tilde{f})\|_F\\
&\leq r\|\mathbf{I}-\eta\hat {\mathbf{A}}\mathcal{H}\|_F^t\|\hat{\mathbf{A}}\nabla \tilde{f} \|_F\\
&\leq r(1+\eta\gamma)^{t}\|\hat {\mathbf{A}} \sum_{j=1}^{L}\tilde{\mathbf{W}}_{j+1:L}\tilde{\mathbf{W}}_{j+1:L}^{\top}\nabla \ell(\tilde{\mathbf{W}})\\
&\tilde{\mathbf{W}}_{1:j-1}^{\top}\tilde{\mathbf{W}}_{1:j-1}\|_F\\
&\leq \kappa^{t}r L\Lambda M^{2L-2}\tau.
\end{align*}
\end{proof}

\begin{lemma}
\begin{equation}
\begin{aligned}
&\|\bm{\delta}(t)\|_F\leq  \kappa^t \Lambda\Big[(\eta\gamma)^{-2}\rho\eta\sigma_{\min}^{\frac{2-2L}{L}}H\cdot\big(l_{\rm {thresh }}+Q)\\
&+4(\eta\gamma)^{-1}\rho Q+2(\eta\gamma)^{-1}(LM^{2L-2}+1)\tau\Big].
\end{aligned}
\end{equation}
\end{lemma}
\begin{proof}
\begin{align*}
&\|\bm{\delta}(t)\|_F=\|\sum_{i=1}^{t}(\mathbf{I}-\eta \hat{\mathbf{A}}\mathcal{H}\big)^{t-i}\hat{\mathbf{A}}\big(\nabla g(\mathbf{W}(i))-\nabla \mathbf{L}(i)\big)\|_F\\
&\leq \sum_{i=1}^{t} (1+\eta\gamma)^{t-i}\Lambda\|\nabla g(\mathbf{W}(i))-\nabla \mathbf{L}(i)\|_F\\
&\leq \sum_{i=1}^{t} (1+\eta\gamma)^{t-i}\Lambda\cdot\Big[\|\nabla g(\mathbf{W}(i))-\nabla \ell(\mathbf{W}(i))\|_F\\
&+\|\nabla \ell(\mathbf{W}(i))-\nabla \mathbf{L}(i)\|_F\Big]\\
&\leq \sum_{i=1}^{t} (1+\eta\gamma)^{t-i}\Lambda \cdot\Big[\frac{\rho}{2}\| \mathbf{W}(i)-\tilde{\mathbf{W}}\|_F+(LM^{2L-2}+1)\tau\Big]\\
&\leq \sum_{i=1}^{t} (1+\eta\gamma)^{t-i}\Lambda \cdot\Big[\frac{\rho}{2} t\eta\sigma_{\min}^{\frac{2-2L}{L}}H\cdot\big(l_{\text {thresh }}+Q)+2\rho Q\\
&+(LM^{2L-2}+1)\tau\Big]\\
&\leq  \kappa^t \Lambda\Big[(\eta\gamma)^{-2}\rho\eta\sigma_{\min}^{\frac{2-2L}{L}}H\cdot\big(l_{\text {thresh }}+Q)+4(\eta\gamma)^{-1}\rho Q\\
&+2(\eta\gamma)^{-1}(LM^{2L-2}+1)\tau\Big].
\end{align*}
\end{proof}

\begin{lemma}
\begin{equation}
\begin{aligned}
\|\bm{\zeta}(t)\|_F\leq2\kappa^tL^2\Lambda^{2} M^{3L-4} B\tau(2\gamma^{-1}+\eta^{-1}r^2).
\end{aligned}
\end{equation}
\end{lemma}
\begin{proof}
\begin{align*}
&\|\bm{\zeta}(t)\|_F=\|\sum_{i=1}^{t}\eta^{-1}(\mathbf{I}-\eta \hat{\mathbf{A}}\mathcal{H}\big)^{t-i}\mathbf{E}(i)+\eta^{-1}\tilde{\mathbf{E}}\|_F\\
&\leq \eta^{-1}\sum_{i=1}^{t} (1+\eta\gamma)^{t-i}\|\mathbf{E}(i)\|+\eta^{-1}\|\tilde{\mathbf{E}}\|_F\\
&\leq 2\eta L^{2}\Lambda^{2} M^{3L-4} B\tau\sum_{i=1}^{t} (1+\eta\gamma)^{t-i}\\
&+2\eta^{-1}r^2 L^{2}\Lambda^{2} M^{3L-4} B\tau\\
&\leq 4\kappa^t\gamma^{-1}L^{2}\Lambda^{2} M^{3L-4} B\tau+2\eta^{-1}r^2 L^{2}\Lambda^{2} M^{3L-4} B\tau\\
&\leq2\kappa^tL^2\Lambda^{2} M^{3L-4} B\tau(2\gamma^{-1}+\eta^{-1}r^2).
\end{align*}
\end{proof}

\begin{lemma}
\begin{equation}
\begin{aligned}
\|\bm{\chi}(t)\|_F&\leq2\kappa^t(\eta\gamma)^{-2}\alpha LM^{2L-2}\tau\Big(\sqrt{\eta\sigma_{\min}^{\frac{2-2L}{L}}H\ell_{\text{thresh}}}\\
&+\sqrt{\eta\sigma_{\min}^{\frac{2-2L}{L}}HQ}  +2\sqrt{Q}\ \Big).
\end{aligned}
\end{equation}
\end{lemma}
\begin{proof}
We can assume that $\hat {\mathbf{A}}(i)$ is satisfied $\alpha$-Lipschitz such as $\|\hat{\mathbf{A}}-\hat{\mathbf{A}}(i)\|_F\leq \alpha\|\mathbf{W}(i)-\tilde{\mathbf{W}}\|_F$.

\begin{align*}
&\|\bm{\chi}(t)\|_F=\|\sum_{i=1}^{t}(\mathbf{I}-\eta \hat{\mathbf{A}}\mathcal{H}\big)^{t-i}\big(\hat{\mathbf{A}}-\hat {\mathbf{A}}(i)\big)\nabla \mathbf{L}(i)\|_F\\
&\leq \alpha LM^{2L-2}\tau\sum_{i=1}^{t} (1+\eta\gamma)^{t-i}\|\mathbf{W}(i)-\tilde{\mathbf{W}}\|_F\\
&\leq \alpha LM^{2L-2}\tau\sum_{i=1}^{t} (1+\eta\gamma)^{t-i}\sqrt{i\eta\sigma_{\min}^{\frac{2-2L}{L}}H\big(\ell_{\text {thresh}}\!+\!Q)\!+\!4Q}\\
&\leq \alpha LM^{2L-2}\tau\sum_{i=1}^{t} (1+\eta\gamma)^{t-i}i\sqrt{\eta\sigma_{\min}^{\frac{2-2L}{L}}H\big(\ell_{\text {thresh}}\!+\!Q)\!+\!4Q}\\
&\leq 2\kappa^t(\eta\gamma)^{-2}\alpha LM^{2L-2}\tau \sqrt{\eta\sigma_{\min}^{\frac{2-2L}{L}}H\big(\ell_{\text {thresh}}\!+\!Q)\!+\!4Q}\\
&\leq 2\kappa^t(\eta\gamma)^{-2}\alpha LM^{2L-2}\tau\Big(\sqrt{\eta\sigma_{\min}^{\frac{2-2L}{L}}H\ell_{\text{thresh}}}\\
&+\sqrt{\eta\sigma_{\min}^{\frac{2-2L}{L}}HQ}+2\sqrt{Q}\ \Big).
\end{align*}
\end{proof}

\begin{lemma}
\begin{equation}
\begin{aligned}
&\|\bm{\iota}(t)\|_F\leq 2\kappa^t(\eta\gamma)^{-1}LM^{2L-2}\Delta\tau.
\end{aligned}
\end{equation}
\end{lemma}
\begin{proof}
\begin{align*}
&\|\bm{\iota}(t)\|_F=\|\sum_{i=1}^{t}(\mathbf{I}-\eta \hat{\mathbf{A}}\mathcal{H}\big)^{t-i}\big(\hat {\mathbf{A}}(i)-\mathbf{A}(i)\big)\nabla \mathbf{L}(i)\|_F\\
&\leq LM^{2L-2}\tau\sum_{i=1}^{t} (1+\eta\gamma)^{t-i}\|\hat {\mathbf{A}}(i)-\mathbf{A}(i)\|_F\\
&\leq LM^{2L-2}\tau\sum_{i=1}^{t} (1+\eta\gamma)^{t-i}\max\limits_{i}\|\hat {\mathbf{A}}(i)-\mathbf{A}(i)\|_F\\
&\leq 2\kappa^t(\eta\gamma)^{-1}LM^{2L-2}\Delta\tau.
\end{align*}
\end{proof}

\begin{lemma}
\begin{equation}
\begin{aligned}
\|\mathbf{W}(t+1)-\tilde{\mathbf{W}}\|_{F}^2 \geq K\kappa^{2t}r^2\nu.
\end{aligned}
\end{equation}
\label{lemmaA13}
\end{lemma}

\begin{proof}
\begin{align*}
&\|\mathbf{W}(t+1)-\tilde{\mathbf{W}}\|_{F}^2 \\
&=\|\bm{u}(t)+\eta\big(\bm{\delta}(t)+\mathbf{d}(t)+\bm{\zeta}(t)+\bm{\chi}(t)+\bm{\iota}(t)\big)\|_{F}^2\\
&\!\geq\! \|\bm{u}(t)\|_{F}^2\!+\!2\eta\big\langle\bm{u}(t),\bm{\delta}(t)\!+\!\mathbf{d}(t)\!+\!\bm{\zeta}(t)\!+\!\bm{\chi}(t)\!+\!\bm{\iota}(t)\big\rangle\\
&\geq \|\bm{u}(t)\|_{F}^2+2\eta\langle \bm{u}(t),\mathbf{d}(t)\rangle   -2\eta\|\bm{u}(t)\|_{F}\|\bm{\delta}(t)\|_{F}\\
&-2\eta\|\bm{u}(t)\|_{F}\|\bm{\zeta}(t)\|_{F}-2\eta\|\bm{u}(t)\|_{F}\|\bm{\chi}(t)\|_{F}\\
&-2\eta \|\bm{u}(t)\|_{F} \|\bm{\iota}(t)\|_{F}\\
&\geq \kappa^{2t}r^2\nu-2\eta\kappa^{t}r L\Lambda M^{2L-2}\tau\Big(\|\bm{\delta}(t)\|_{F}+\|\bm{\zeta}(t)\|_{F}\\
&+\|\bm{\chi}(t)\|_{F}+\|\bm{\iota}(t)\|_{F}\Big).\\
\end{align*}

According to the above auxiliary results, we get

\begin{equation}
\label{Eq6}
\begin{aligned}
&\|\mathbf{W}(t+1)-\tilde{\mathbf{W}}\|_{F}^2 \\
&\geq \kappa^{2t}r\Bigg\lbrace
r\nu-2\eta L\Lambda M^{2L-2}\tau\Big[
\Lambda\Big((\eta\gamma)^{-2}\rho\eta\sigma_{\min}^{\frac{2-2L}{L}}H\\
&\cdot\big(l_{\rm {thresh }}+Q)+4(\eta\gamma)^{-1}\rho Q+2(\eta\gamma)^{-1}(LM^{2L-2}+1)\tau\Big)\\
&+2L^2\Lambda^{2} M^{3L-4} B\tau(2\gamma^{-1}+\eta^{-1}r^2)\\
&+2(\eta\gamma)^{-2}\alpha LM^{2L-2}\tau\Big(\sqrt{\eta\sigma_{\min}^{\frac{2-2L}{L}}Hl_{\text{thresh}}}+\sqrt{\eta\sigma_{\min}^{\frac{2-2L}{L}}HQ}\\
&+2\sqrt{Q}\ \Big)+2(\eta\gamma)^{-1}LM^{2L-2}\Delta\tau\Big]\Bigg\rbrace.
\end{aligned}
\end{equation}

In order to contradict the Lemma \rm{\ref{lemmaA4}}, we require the sum of the last inequality in (\ref{Eq6}) to be positive. Using $Kr\nu$ to constrain the eleven items in the square bracket, some parameter ranges can be obtained. We set $K=\frac{1}{11}$, the derivation process is as follows. 

Firstly, five terms can be discussed in the $\bm{\delta}(t)$ as follows
\begin{equation}
\begin{aligned}
2\eta L\Lambda M^{2L-2}&\tau\cdot\Lambda(\eta\gamma)^{-2}\rho\eta\sigma_{\min}^{\frac{2-2L}{L}}Hl_{\rm {thresh }}\leq Kr\nu\\
&\Leftrightarrow l_{\rm {thresh }}\leq \frac{Kr\nu \gamma^2}{2L\Lambda M^{2L-2}\tau\rho\sigma_{\min}^{\frac{2-2L}{L}}H}\\
&\Leftrightarrow l_{\rm {thresh }}\leq O(r\gamma^2\tau),
\end{aligned}
\end{equation}

\begin{equation}
\begin{aligned}
2\eta L\Lambda& M^{2L-2}\tau\cdot\Lambda(\eta\gamma)^{-2}\rho\eta\sigma_{\min}^{\frac{2-2L}{L}}HQ\leq Kr\nu\\
&\Leftrightarrow r\leq \frac{K\nu\gamma^2}{16L^5\Lambda^6M^{8L-10}B^2\tau^3\rho\sigma_{\min}^{\frac{2-2L}{L}}H}\\
&\Leftrightarrow r\leq O(\gamma^2\tau^{-3}),
\end{aligned}
\end{equation}

\begin{equation}
\begin{aligned}
2\eta L&\Lambda^2 M^{2L-2}\tau\cdot4(\eta\gamma)^{-1}\rho Q\leq Kr\nu\\
&\Leftrightarrow r\leq \frac{K\nu\gamma}{64L^5\Lambda^6M^{8L-10}B^2\tau^3\rho}\\
&\Leftrightarrow r\leq o(\gamma\tau^{-3}),
\end{aligned}
\end{equation}

\begin{equation}
\begin{aligned}
2\eta L\Lambda^2 M^{2L-2}\tau\cdot&2(\eta\gamma)^{-1}LM^{2L-2}\tau\leq Kr\nu\\
&\Leftrightarrow r\geq \frac{4L^2\Lambda^2M^{4L-4}\tau^2}{K\nu\gamma}\\
&\Leftrightarrow r\geq O(\gamma^{-1}\tau^2),
\end{aligned}
\end{equation}

\begin{equation}
\begin{aligned}
2\eta L\Lambda^2 &M^{2L-2}\tau\cdot2(\eta\gamma)^{-1}\tau\leq Kr\nu\\
&\Leftrightarrow r\geq \frac{4L\Lambda^2M^{2L-2}\tau^2}{K\nu\gamma}\\
&\Leftrightarrow r\geq O(\gamma^{-1}\tau^2).
\end{aligned}
\end{equation}

Then, two terms can be discussed in the $\bm{\zeta}(t)$ as follows

\begin{equation}
\begin{aligned}
2\eta L\Lambda M^{2L-2}\tau\cdot2L^2&\Lambda^{2} M^{3L-4} B\tau\cdot2\gamma^{-1}\leq Kr\nu\\
&\Leftrightarrow \eta\leq \frac{Kr\nu\gamma}{8L^3\Lambda^3M^{5L-6}B\tau^2}\\
&\Leftrightarrow \eta\leq O(r\gamma\tau^{-2}),
\end{aligned}
\end{equation}

\begin{equation}
\begin{aligned}
2\eta L\Lambda M^{2L-2}\tau\cdot2L^2&\Lambda^{2} M^{3L-4} B\tau\eta^{-1}r^2\leq Kr\nu\\
&\Leftrightarrow r\leq \frac{K\nu}{4L^3\Lambda^3M^{5L-6}B\tau^2}\\
&\Leftrightarrow r\leq O(\tau^{-2}).
\end{aligned}
\end{equation}

Next, three terms can be discussed in the $\bm{\chi}(t)$ as follows

\begin{equation}
\begin{aligned}
2\eta L\Lambda M^{2L-2}\tau\cdot2(\eta\gamma&)^{-2}\alpha LM^{2L-2}\tau\\
&\cdot\sqrt{\eta\sigma_{\min}^{\frac{2-2L}{L}}Hl_{\text{thresh}}}\leq Kr\nu\\
\Leftrightarrow l_{\text{thresh}} \leq& \frac{K^2r^2\nu^2\eta\gamma^4}{16\alpha^2 L^4\Lambda^2 M^{8L-8}\sigma_{\min}^{\frac{2-2L}{L}}H\tau^4}\\
\Leftrightarrow l_{\text{thresh}} \leq& O(r^2\eta\gamma^4\tau^{-4}),
\end{aligned}
\end{equation}

\begin{equation}
\begin{aligned}
2\eta L\Lambda M^{2L-2}\tau\cdot2&(\eta\gamma)^{-2}\alpha LM^{2L-2}\tau\\
\cdot&\sqrt{\eta\sigma_{\min}^{\frac{2-2L}{L}}HQ}\leq Kr\nu\\
\Leftrightarrow \eta \geq &\frac{128\alpha^2L^8\Lambda^6 M^{14L-16}B^2\sigma_{\min}^{\frac{2-2L}{L}}H\tau^6}{K^2\nu^2\gamma^4}\\
\Leftrightarrow \eta \geq& O(\gamma^{-4}\tau^6),
\end{aligned}
\end{equation}

\begin{equation}
\begin{aligned}
2\eta L\Lambda M^{2L-2}\tau\cdot2&(\eta\gamma)^{-2}\alpha LM^{2L-2}\tau\cdot2\sqrt{Q}\leq Kr\nu\\
&\Leftrightarrow\eta\geq \frac{8\sqrt{2}\alpha L^4\Lambda^3M^{7L-8}B\tau^3}{K\nu\gamma^2}\\
&\Leftrightarrow\eta\geq O(\gamma^{-4}\tau^6).
\end{aligned}
\end{equation}

Finally, we have one term in the $\bm{\iota}(t)$

\begin{equation}
\begin{aligned}
2\eta L\Lambda M^{2L-2}\tau\cdot2(\eta\gamma)^{-1}&LM^{2L-2}\Delta\tau\leq Kr\nu\\
\Leftrightarrow \Delta &\leq \frac{Kr\nu\gamma}{4L^2M^{4L-4}\Lambda\tau^2	}\\
\Leftrightarrow \Delta &\leq O(r\gamma\tau^{-2}).
\end{aligned}
\end{equation}

Since $\gamma=o(\tau^{\frac{1}{2}})$, all parameters are related to $\tau$. In order to satisfy the above eleven inequalities, each parameter is finally determined as $\eta=O(\tau^2)$, $r=O(\tau)$ and $l_{\text{thresh}}=O(\tau^2)$. Meanwhile, the specific values of each parameter will be given below.
\end{proof}

We have proved Lemma \rm{\ref{lemmaA13}}, which shows that the lower bound of $\|\mathbf{W}(t+1)-\tilde{\mathbf{W}}\|_{F}^2$ increases exponentially. When $t_{\text{thresh}}=\frac{\omega}{\eta\gamma}$ is large enough, there will be
\begin{equation}
\begin{aligned}
\label{deqn_ex8}
K\kappa^{2t}r^2\nu\geq t\eta\sigma_{\min}^{\frac{2-2L}{L}}H\cdot\big(l_{\rm {thresh}}+Q)+4Q,
\end{aligned}
\end{equation}
which is in contradiction with Lemma \rm{\ref{lemmaA4}}. Hence, we can deduce 
\begin{equation}
 \ell(\mathbf{W}(t))-\ell(\tilde{\mathbf{W}}) \leq -\ell_{\text {thresh}}.
\end{equation}
Here, the proof of Lemma 2 is completed.

\section*{Proof of Lemma 3}

We need to confirm that the increase of function value in each $t_{\text {thresh}}$ iterations is bounded. According to Corollary \rm{\ref{corollaryA2}},
\begin{equation}
\begin{aligned}
\ell(\mathbf{W}(t+1))-\ell(\mathbf{W}(t)) \leq 8r^2L^4\Lambda^4M^{6L-8}B^2\tau^2.
\end{aligned}
\end{equation}
Replacing the parameters $r$ and $\ell_{\text {thresh}}$ in Table \ref{table1} with the following formula,  we have 
\begin{equation}
\begin{aligned}
8r^2L^4\Lambda^4M^{6L-8}B^2\tau^2&\leq \frac{\delta l_{\rm{thresh}} }{4}\\
\Leftrightarrow\frac{K^2\nu^2\gamma^{8}\delta^2}{8L^{6}\Lambda^{6}M^{4L-2} B^{2}\rho^2}&\leq \frac{K^2\nu^2\gamma^{6}\delta^2}{8L^{6}\Lambda^{6}M^{4L-2} B^{2}\rho^2\tau}\\
\Leftrightarrow\gamma^2\tau&\leq1.
\end{aligned}
\end{equation}
Hence, we can get 
\begin{equation}
\begin{aligned}
\ell(\mathbf{W}(t+1))-\ell(\mathbf{W}(t)) \leq \frac{\delta l_{\rm{thresh}}}{4},
\end{aligned}
\end{equation}
to further average $t_{\text {thresh}}$ as follows
\begin{equation}
\label{Eq7}
\begin{aligned}
\frac{\ell(\mathbf{W}(t+1))-\ell(\mathbf{W}(t))}{t_{\text {thresh}}} \leq \frac{\delta g_{\rm{thresh}}}{4}.
\end{aligned}
\end{equation}

According to Lemma \rm{\ref{lemmaA3}},
\begin{equation}
\begin{aligned}
\ell(\mathbf{W}(t+1))-\ell(\mathbf{W}(t)) \leq 8\eta^2L^4\Lambda^4M^{6L-8}B^2\tau^2.
\end{aligned}
\end{equation}
Applying to (\ref{Eq7}) with upper bound $\frac{\delta g_{\rm{thresh}} }{4}$,
\begin{equation}
\begin{aligned}
8\eta^2L^4\Lambda^4M^{6L-8}B^2\tau^2&\leq \frac{\delta g_{\rm{thresh}} }{4}\\
\Leftrightarrow8\eta^2L^4\Lambda^4M^{6L-8}B^2\tau^2&\leq \frac{\delta l_{\rm{thresh}} }{4}\cdot \frac{\eta\gamma}{\omega}\\
\Leftrightarrow8\eta L^4\Lambda^4M^{6L-8}B^2\tau^2&\leq \frac{\delta l_{\rm{thresh}} }{4}\cdot \frac{\gamma}{\omega}\\
\Leftrightarrow\frac{K^{2} \nu^{2}\tau^2 \gamma^{6} \delta^{2}}{8 L^{6} \Lambda^{6}M^{4L-2} B^{2}\rho^{2}\tau^2 \omega^{2}}&\leq \frac{K^2\nu^2\gamma^{7}\delta^2}{8L^{6}\Lambda^{6}M^{4L-2} B^{2}\rho^2\tau\omega}\\
\Leftrightarrow\frac{\tau}{\gamma}&\leq\omega,
\end{aligned}
\end{equation}
Which proves the Lemma 3 completely.

\begin{table}[]
\begingroup
\renewcommand{\arraystretch}{1.65}
\caption{Parameters of Algorithm 1}
\label{table1}
\begin{tabular}{ccc}
\hline
Parameter & Value & Dependency on $\tau$    \\ \hline
$\eta$    & $\frac{K^{2} \nu^{2} \gamma^{6} \delta^{2}}{64 L^{10} \Lambda^{10}M^{10L-10} B^{4}\rho^{2}\tau^2 \omega^{2}}$     & $O(\tau^{2})$           \\
$r$       & $\frac{K\nu\gamma^{4}\delta}{8L^{5}\Lambda^{5}M^{5L-5} B^{2}\rho\tau}$     & $O(\tau)$               \\
$\omega$  & $\tau^{-\frac{1}{2}}$     & $O(\tau^{-\frac{1}{2}})$ \\
$\ell_{\text {thresh}}$       & $\frac{K^2\nu^2\gamma^{6}\delta}{2L^{6}\Lambda^{6}M^{4L-2} B^{2}\rho^2\tau}$     & $O(\tau^{2})$           \\
$t_{\text {thresh}}$       & $\frac{\omega}{\eta\gamma}$     & $O(\tau^{-3})$          \\
$g_{\text {thresh}}$       & $\frac{\ell_{\text {thresh}}}{t_{\text {thresh}}}$     & $O(\tau^{5})$           \\ \hline
\end{tabular}
\endgroup
\end{table}
\section*{Proof of Theorem 1}

By integrating Lemmas 1, 2 and 3, the following result is obtained

\begin{equation}
\label{Eq8}
\begin{aligned}
\begin{cases}
\ell(\mathbf{W}(t\!+\!1))-\ell(\mathbf{W}(t))\leq-g_{\rm{thresh}},\ \mathbf{W}(t)\in\Omega_t;\\
\ell(\mathbf{W}(t\!+\!1))-\ell(\mathbf{W}(t))\leq \dfrac{\delta g_{\rm{thresh}}}{2},\ \mathbf{W}(t)\in\Omega_t^c.
\end{cases}
\end{aligned}
\end{equation}

(\ref{Eq8}) quantitatively analyzes the variation of  function value $\ell$ at the iteration. Before reaching the second-order critical point, the function value $\ell$ at $\Omega_t$ is in a decreasing state. After a couple of iterations, the change of the function value  $\ell$ tends to be stable and can fluctuate within a small range. According to the method in \cite{RN37}, we use the law of total expectation to establish a high probability upper bound to solve the iteration time, so as to avoid the interdependence of random variables $\mathbf{W}(t)$. We assume that the probability of $\Omega_t$ is $P_t$,
\begin{equation}
\begin{aligned}
&\ell(\mathbf{W}(t+1))-\ell(\mathbf{W}(t))\\
&=\big[\ell(\mathbf{W}(t+1))-\ell(\mathbf{W}(t))|\Omega_t\big]P_t+\big[\ell(\mathbf{W}(t+1))\\
&-\ell(\mathbf{W}(t))|\Omega_t^c\big](1-P_t)\\
&\leq -g_{\rm{thresh}}P_t+\frac{\delta g_{\rm{thresh}}}{2}(1-P_t)\\
&\leq\frac{\delta g_{\rm{thresh}}}{2}-g_{\rm{thresh}}P_t,
\end{aligned}
\end{equation}
further setting $T$ as the total iteration time,
\begin{equation}
\begin{aligned}
\frac{1}{T}\sum_{t=1}^T\big[\ell(\mathbf{W}(t+1))&-\ell(\mathbf{W}(t))\big]\\
&\leq \frac{1}{T}\sum_{t=1}^T\big( \frac{\delta g_{\rm{thresh}}}{2}-g_{\rm{thresh}}P_t \big)\\
\Leftrightarrow \frac{1}{T}\sum_{t=1}^T P_t \leq \frac{\delta}{2}&+\frac{\ell(\mathbf{W}(0))-\ell(\mathbf{W}^*)}{T g_{\rm{thresh}}} \leq \delta.\\
\end{aligned}
\end{equation}

By limiting the last formula to a small upper bound $\delta$, we can obtain a second-order critical point with high probability at least $1-\delta$ as follow
\begin{equation}
\begin{aligned}
\frac{1}{T}\sum_{t=1}^T (1-P_t) \geq 1-\delta,
\end{aligned}
\end{equation}
where we can choose $T\geq \dfrac{c\big(\ell(\mathbf{W}(0))-\ell(\mathbf{W}^*)\big)}{\delta g_{\rm{thresh}}}$.
Finally, by adding the parameter values in Table \ref{table1} into $T$, we complete the proof of Theorem 1.
}

\nocite{*}
\bibliographystyle{IEEEtran} 
\bibliography{references}  

\begin{thebibliography}{10}
\providecommand{\url}[1]{#1}
\csname url@samestyle\endcsname
\providecommand{\newblock}{\relax}
\providecommand{\bibinfo}[2]{#2}
\providecommand{\BIBentrySTDinterwordspacing}{\spaceskip=0pt\relax}
\providecommand{\BIBentryALTinterwordstretchfactor}{4}
\providecommand{\BIBentryALTinterwordspacing}{\spaceskip=\fontdimen2\font plus
\BIBentryALTinterwordstretchfactor\fontdimen3\font minus
  \fontdimen4\font\relax}
\providecommand{\BIBforeignlanguage}[2]{{%
\expandafter\ifx\csname l@#1\endcsname\relax
\typeout{** WARNING: IEEEtran.bst: No hyphenation pattern has been}%
\typeout{** loaded for the language `#1'. Using the pattern for}%
\typeout{** the default language instead.}%
\else
\language=\csname l@#1\endcsname
\fi
#2}}
\providecommand{\BIBdecl}{\relax}
\BIBdecl

\bibitem{RN85}
A.~Voulodimos, N.~Doulamis, A.~Doulamis, and E.~Protopapadakis, ``Deep learning
  for computer vision: A brief review,'' \emph{Computational Intelligence and
  Neuroscience}, vol. 2018, p. 7068349, 2018.

\bibitem{RN110}
A.~Esteva, K.~Chou, S.~Yeung, N.~Naik, A.~Madani, A.~Mottaghi, Y.~Liu,
  E.~Topol, J.~Dean, and R.~Socher, ``Deep learning-enabled medical computer
  vision,'' \emph{NPJ Digital Medicine}, vol.~4, no.~1, pp. 1--9, 2021.

\bibitem{RN111}
N.~O’Mahony, S.~Campbell, A.~Carvalho, S.~Harapanahalli, G.~V. Hernandez,
  L.~Krpalkova, D.~Riordan, and J.~Walsh, ``Deep learning vs. traditional
  computer vision,'' in \emph{Proc. Computer Vision Conference (CVC)}, 2019,
  pp. 128--144.

\bibitem{RN87}
D.~W. Otter, J.~R. Medina, and J.~K. Kalita, ``A survey of the usages of deep
  learning for natural language processing,'' \emph{IEEE Transactions on Neural
  Networks and Learning Systems}, vol.~32, no.~2, pp. 604--624, 2020.

\bibitem{RN227}
A.~Galassi, M.~Lippi, and P.~Torroni, ``Attention in natural language
  processing,'' \emph{IEEE Transactions on Neural Networks and Learning
  Systems}, vol.~32, no.~10, pp. 4291--4308, 2020.

\bibitem{RN113}
T.~Young, D.~Hazarika, S.~Poria, and E.~Cambria, ``Recent trends in deep
  learning based natural language processing,'' \emph{IEEE Computational
  Intelligence Magazine}, vol.~13, no.~3, pp. 55--75, 2018.

\bibitem{RN115}
B.~Lim and S.~Zohren, ``Time-series forecasting with deep learning: a survey,''
  \emph{Philosophical Transactions of the Royal Society A}, vol. 379, no. 2194,
  p. 20200209, 2021.

\bibitem{RN118}
J.~F. Torres, D.~Hadjout, A.~Sebaa, F.~Martínez-Álvarez, and A.~Troncoso,
  ``Deep learning for time series forecasting: a survey,'' \emph{Big Data},
  vol.~9, no.~1, pp. 3--21, 2021.

\bibitem{RN88}
H.~Wang, Z.~Lei, X.~Zhang, B.~Zhou, and J.~Peng, ``A review of deep learning
  for renewable energy forecasting,'' \emph{Energy Conversion Management}, vol.
  198, p. 111799, 2019.

\bibitem{RN122}
Q.~Yang, Z.~Wang, K.~Guo, C.~Cai, and X.~Qu, ``Physics-driven synthetic data
  learning for biomedical magnetic resonance: The imaging physics-based data
  synthesis paradigm for artificial intelligence,'' \emph{IEEE Signal
  Processing Magazine}, 2022, doi: 10.1109/MSP.2022.3183809.

\bibitem{RN228}
Z.~Wang, D.~Guo, Z.~Tu, Y.~Huang, Y.~Zhou, J.~Wang, L.~Feng, D.~Lin, Y.~You,
  and T.~Agback, ``A sparse model-inspired deep thresholding network for
  exponential signal reconstruction-application in fast biological
  spectroscopy,'' \emph{IEEE transactions on Neural Networks and Learning
  Systems}, 2022, doi: 10.1109/TNNLS.2022.3144580.

\bibitem{RN100}
X.~Qu, Y.~Huang, H.~Lu, T.~Qiu, D.~Guo, T.~Agback, V.~Orekhov, and Z.~Chen,
  ``Accelerated nuclear magnetic resonance spectroscopy with deep learning,''
  \emph{Angewandte Chemie International Edition}, vol. 132, no.~26, pp.
  10\,383--10\,386, 2020.

\bibitem{RN121}
D.~Chen, Z.~Wang, D.~Guo, V.~Orekhov, and X.~Qu, ``Review and prospect: Deep
  learning in nuclear magnetic resonance spectroscopy,'' \emph{Chemistry–A
  European Journal}, vol.~26, no.~46, pp. 10\,391--10\,401, 2020.

\bibitem{RN133}
S.~Ruder, ``An overview of gradient descent optimization algorithms,''
  \emph{arXiv preprint arXiv:.1609.04747}, 2016.

\bibitem{RN59}
A.~Fawzi, S.-M. Moosavi-Dezfooli, and P.~Frossard, ``The robustness of deep
  networks: A geometrical perspective,'' \emph{IEEE Signal Processing
  Magazine}, vol.~34, no.~6, pp. 50--62, 2017.

\bibitem{RN229}
S.-B. Lin, ``Generalization and expressivity for deep nets,'' \emph{IEEE
  Transactions on Neural Networks and Learning Systems}, vol.~30, no.~5, pp.
  1392--1406, 2018.

\bibitem{RN60}
X.~Bai, X.~Wang, X.~Liu, Q.~Liu, J.~Song, N.~Sebe, and B.~Kim, ``Explainable
  deep learning for efficient and robust pattern recognition: A survey of
  recent developments,'' \emph{Pattern Recognition}, vol. 120, p. 108102, 2021.

\bibitem{RN63}
R.~Sun, ``Optimization for deep learning: An overview,'' \emph{Journal of the
  Operations Research Society of China}, vol.~8, no.~2, pp. 249--294, 2020.

\bibitem{RN62}
C.~Zhang, S.~Bengio, M.~Hardt, B.~Recht, and O.~Vinyals, ``Understanding deep
  learning (still) requires rethinking generalization,'' \emph{Communications
  of the ACM}, vol.~64, no.~3, pp. 107--115, 2021.

\bibitem{RN92}
T.~Poggio, A.~Banburski, and Q.~Liao, ``Theoretical issues in deep networks,''
  \emph{Proceedings of the National Academy of Sciences}, vol. 117, no.~48, pp.
  30\,039--30\,045, 2020.

\bibitem{RN93}
B.~Neyshabur, R.~Tomioka, and N.~Srebro, ``In search of the real inductive
  bias: On the role of implicit regularization in deep learning,'' in
  \emph{Proc. Workshop Contribution at International Conference on Learning
  Representations (ICLR)}, 2015, pp. 1--9.

\bibitem{RN41}
N.~Razin and N.~Cohen, ``Implicit regularization in deep learning may not be
  explainable by norms,'' in \emph{Proc. Advances in Neural Information
  Processing Systems (NIPS)}, 2020, pp. 21\,174--21\,187.

\bibitem{RN40}
S.~Arora, N.~Cohen, W.~Hu, and Y.~Luo, ``Implicit regularization in deep matrix
  factorization,'' in \emph{Proc. Advances in Neural Information Processing
  Systems (NIPS)}, 2019, pp. 7413--7424.

\bibitem{RN52}
G.~Gidel, F.~Bach, and S.~Lacoste-Julien, ``Implicit regularization of discrete
  gradient dynamics in linear neural networks,'' in \emph{Proc. Advances in
  Neural Information Processing Systems (NIPS)}, 2019, pp. 3202--3211.

\bibitem{RN43}
A.~M. Saxe, J.~L. McClelland, and S.~Ganguli, ``Exact solutions to the
  nonlinear dynamics of learning in deep linear neural networks,'' \emph{arXiv
  preprint arXiv:1312.6120}, 2013.

\bibitem{RN45}
------, ``A mathematical theory of semantic development in deep neural
  networks,'' \emph{Proceedings of the National Academy of Sciences}, vol. 116,
  no.~23, pp. 11\,537--11\,546, 2019.

\bibitem{RN84}
J.-F. Cai, E.~J. Candès, and Z.~Shen, ``A singular value thresholding
  algorithm for matrix completion,'' \emph{SIAM Journal on optimization},
  vol.~20, no.~4, pp. 1956--1982, 2010.

\bibitem{RN98}
Y.~Huang, J.~Zhao, Z.~Wang, V.~Orekhov, D.~Guo, and X.~Qu, ``Exponential signal
  reconstruction with deep {Hankel} matrix factorization,'' \emph{IEEE
  Transactions on Neural Networks and Learning Systems}, 2021, doi:
  10.1109/TNNLS.2021.3134717.

\bibitem{RN97}
X.~Zhang, H.~Lu, D.~Guo, Z.~Lai, H.~Ye, X.~Peng, B.~Zhao, and X.~Qu,
  ``Accelerated {MRI} reconstruction with separable and enhanced low-rank
  {Hankel} regularization,'' \emph{IEEE Transactions on Medical Imaging},
  vol.~41, no.~9, pp. 2486--2498, 2022.

\bibitem{RN99}
X.~Zhang, D.~Guo, Y.~Huang, Y.~Chen, L.~Wang, F.~Huang, Q.~Xu, and X.~Qu,
  ``Image reconstruction with low-rankness and self-consistency of k-space data
  in parallel {MRI},'' \emph{Medical Image Analysis}, vol.~63, p. 101687, 2020.

\bibitem{RN134}
Z.~Wang, C.~Qian, D.~Guo, H.~Sun, R.~Li, B.~Zhao, and X.~Qu, ``One-dimensional
  deep low-rank and sparse network for accelerated {MRI},'' \emph{IEEE
  Transactions on Medical Imaging}, 2022, doi: 10.1109/TMI.2022.3203312.

\bibitem{RN4}
Z.~Li, Y.~Luo, and K.~Lyu, ``Towards resolving the implicit bias of gradient
  descent for matrix factorization: Greedy low-rank learning,'' \emph{arXiv
  preprint arXiv:2012.09839}, 2020.

\bibitem{RN34}
A.~Jacot, F.~Ged, B.~Şimşek, C.~Hongler, and F.~Gabriel, ``Saddle-to-saddle
  dynamics in deep linear networks: Small initialization training, symmetry,
  and sparsity,'' \emph{arXiv preprint arXiv:2106.15933}, 2022.

\bibitem{RN44}
D.~Gissin, S.~Shalev-Shwartz, and A.~Daniely, ``The implicit bias of depth: How
  incremental learning drives generalization,'' \emph{arXiv preprint
  arXiv:1909.12051}, 2019.

\bibitem{RN46}
H.-H. Chou, C.~Gieshoff, J.~Maly, and H.~Rauhut, ``Gradient descent for deep
  matrix factorization: Dynamics and implicit bias towards low rank,''
  \emph{arXiv preprint arXiv:2011.13772}, 2020.

\bibitem{RN21}
C.~Jin, R.~Ge, P.~Netrapalli, S.~M. Kakade, and M.~I. Jordan, ``How to escape
  saddle points efficiently,'' in \emph{Proc. International Conference on
  Machine Learning (ICML)}, 2017, pp. 1724--1732.

\bibitem{RN57}
A.~Anandkumar and R.~Ge, ``Efficient approaches for escaping higher order
  saddle points in non-convex optimization,'' in \emph{Proc. 29th Conference on
  Learning Theory (COLT)}, vol.~49, Conference Proceedings, pp. 81--102.

\bibitem{RN23}
R.~Ge, F.~Huang, C.~Jin, and Y.~Yuan, ``Escaping from saddle points-online
  stochastic gradient for tensor decomposition,'' in \emph{Proc. 28th
  Conference on Learning Theory (COLT)}, 2015, pp. 797--842.

\bibitem{RN128}
T.~Tieleman and G.~Hinton, ``Lecture 6.5-{RMSProp}: Divide the gradient by a
  running average of its recent magnitude,'' \emph{COURSERA: Neural Networks
  for Machine Learning}, vol.~4, no.~2, pp. 26--31, 2012.

\bibitem{RN38}
M.~Staib, S.~Reddi, S.~Kale, S.~Kumar, and S.~Sra, ``Escaping saddle points
  with adaptive gradient methods,'' in \emph{Proc. International Conference on
  Machine Learning (ICML)}, 2018, pp. 5956--5965.

\bibitem{RN31}
K.~Kawaguchi, ``Deep learning without poor local minima,'' in \emph{Proc.
  Advances in Neural Information Processing Systems (NIPS)}, 2016, pp.
  586--594.

\bibitem{RN65}
C.~Yun, S.~Sra, and A.~Jadbabaie, ``Global optimality conditions for deep
  neural networks,'' \emph{arXiv preprint arXiv:1707.02444}, 2017.

\bibitem{RN19}
S.~Arora, N.~Cohen, and E.~Hazan, ``On the optimization of deep networks:
  Implicit acceleration by overparameterization,'' in \emph{Proc. International
  Conference on Machine Learning (ICML)}, 2018, pp. 244--253.

\bibitem{RN14}
S.~Arora, N.~Cohen, N.~Golowich, and W.~Hu, ``A convergence analysis of
  gradient descent for deep linear neural networks,'' in \emph{Proc.
  International Conference on Learning Representations (ICLR)}, 2019, pp.
  1--19.

\bibitem{RN37}
H.~Daneshmand, J.~Kohler, A.~Lucchi, and T.~Hofmann, ``Escaping saddles with
  stochastic gradients,'' in \emph{Proc. International Conference on Machine
  Learning (ICML)}, 2018, pp. 1155--1164.

\bibitem{RN51}
S.~J. Reddi, S.~Kale, and S.~Kumar, ``On the convergence of adam and beyond,''
  in \emph{Proc. International Conference on Learning Representations (ICLR)},
  2018, pp. 1--23.

\bibitem{RN82}
Y.~Nesterov, \emph{Introductory {Lectures} on {Convex} {Optimization}: A
  {Basic} {Course}}.\hskip 1em plus 0.5em minus 0.4em\relax Springer Science \&
  Business Media, 2003, vol.~87.

\bibitem{RN26}
C.~Jin, P.~Netrapalli, R.~Ge, S.~M. Kakade, and M.~I. Jordan, ``On nonconvex
  optimization for machine learning: Gradients, stochasticity, and saddle
  points,'' \emph{Journal of the ACM}, vol.~68, no.~2, pp. 1--29, 2021.

\bibitem{RN95}
N.~Srebro, ``Learning with matrix factorizations,'' Ph.D. dissertation,
  Massachusetts Institute of Technology, 2004.

\bibitem{RN73}
J.~Ying, J.-F. Cai, D.~Guo, G.~Tang, Z.~Chen, and X.~Qu, ``Vandermonde
  factorization of hankel matrix for complex exponential signal
  recovery-application in fast {NMR} spectroscopy,'' \emph{IEEE Transactions on
  Signal Processing}, vol.~66, no.~21, pp. 5520--5533, 2018.

\bibitem{RN124}
D.~H. Johnson, ``Signal-to-noise ratio,'' \emph{Scholarpedia}, vol.~1, no.~12,
  p. 2088, 2006.

\end{thebibliography}

\end{document}